\documentclass[sigconf]{acmart}
\usepackage{hyperref}
\usepackage{multirow} 
\usepackage{threeparttable}
\usepackage{enumitem}
\usepackage[normalem]{ulem}
\usepackage{subcaption}
\usepackage{graphicx}
\usepackage{enumitem}
\usepackage{pifont}
\usepackage{array}
\usepackage{float}
\usepackage{placeins}
\usepackage{amsthm}
\usepackage[bordercolor=white,backgroundcolor=gray!30,linecolor=black,colorinlistoftodos]{todonotes}

\def\Asym{\tilde{\boldsymbol{A}}_{\mathrm{sym}}}
\def\Ars{\tilde{\boldsymbol{A}}_{\mathrm{rs}}}
\def\Acs{\tilde{\boldsymbol{A}}_{\mathrm{cs}}}
\def\A{\tilde{\boldsymbol{A}}}
\def\mA{\mathbb{A}}

\newtheorem{proposition}{Proposition}[section]
\newtheorem*{proposition*}{Proposition}
\newtheorem*{theorem*}{Theorem}
\AtBeginDocument{%
  \providecommand\BibTeX{{%
    \normalfont B\kern-0.5em{\scshape i\kern-0.25em b}\kern-0.8em\TeX}}}

\copyrightyear{2024}
\acmYear{2024}
\setcopyright{acmlicensed}
\acmConference[KDD '24] {Proceedings of the 30th ACM SIGKDD Conference on Knowledge Discovery and Data Mining }{August 25--29, 2024}{Barcelona, Spain.}
\acmBooktitle{Proceedings of the 30th ACM SIGKDD Conference on Knowledge Discovery and Data Mining (KDD '24), August 25--29, 2024, Barcelona, Spain}
\acmDOI{10.1145/3637528.3671946}
\acmISBN{979-8-4007-0490-1/24/08}

\begin{document} 
\setlist[itemize]{leftmargin=0.6cm}

\title{Heuristic Learning with Graph Neural Networks: A Unified Framework for Link Prediction}

\author{Juzheng Zhang}
\affiliation{
  \institution{Department of Electronic Engineering, \\Tsinghua University}
  \city{Beijing}
  \country{China}
}

\author{Lanning Wei}
\affiliation{
  \institution{Department of Electronic Engineering, \\Tsinghua University}
  \city{Beijing}
  \country{China}
}

\author{Zhen Xu}
\affiliation{
  \institution{Department of Electronic Engineering, \\Tsinghua University}
  \city{Beijing}
  \country{China}
}

\author{Quanming Yao}
\authornote{Correspondence is to Q. Yao at \href{mailto:qyaoaa@tsinghua.edu.cn}{qyaoaa@tsinghua.edu.cn}.}
\affiliation{
  \institution{Department of Electronic Engineering, \\Tsinghua University}
  \city{Beijing}
  \country{China}
}


\begin{abstract}
Link prediction is a fundamental task in graph learning, inherently shaped by the topology of the graph. While traditional heuristics are grounded in graph topology, they encounter challenges in generalizing across diverse graphs. Recent research efforts have aimed to leverage the potential of heuristics, yet a unified formulation accommodating both local and global heuristics remains undiscovered. Drawing insights from the fact that both local and global heuristics can be represented by adjacency matrix multiplications, we propose a unified matrix formulation to accommodate and generalize various heuristics.
We further propose the Heuristic Learning Graph Neural Network (HL-GNN) to efficiently implement the formulation. HL-GNN adopts intra-layer propagation and inter-layer connections, allowing it to reach a depth of around 20 layers with lower time complexity than GCN. 
Extensive experiments on the Planetoid, Amazon, and OGB datasets underscore the effectiveness and efficiency of HL-GNN. It outperforms existing methods by a large margin in prediction performance. Additionally, HL-GNN is several orders of magnitude faster than heuristic-inspired methods while requiring only a few trainable parameters. The case study further demonstrates that the generalized heuristics and learned weights are highly interpretable. The code is available at \url{https://github.com/LARS-research/HL-GNN}.

\end{abstract}





\maketitle

\section{Introduction}
Link prediction stands as a cornerstone in the domain of graph machine learning, facilitating diverse applications from knowledge graph reasoning~\cite{redgnn} to drug interaction prediction~\cite{wang2024accurate,zhang2023emerging} and recommender systems~\cite{matrix_factorization}. While its significance is unquestionable, research in this area has not reached the same depth as that for node or graph classification~\cite{dropedge,nodeformer,li2019semi,pooling}.

In graph machine learning, two fundamental sources of information play a pivotal role: node features and graph topology~\cite{GPRGNN,GCN}. 
Link prediction task is inherently driven by graph topology~\cite{survey1,SEAL,Neo-GNN}. Heuristics, which derive exclusively from graph topology, naturally align with link prediction. The appeal of heuristics lies in their simplicity and independence of learning. 
While heuristics are crafted from human intuition and insights, they can be broadly categorized into two types: local heuristics, which focus on neighboring nodes, and global heuristics, which focus on global paths~\cite{survey1}.

Effective link prediction benefits from both local and global topological information~\cite{survey2}. For instance, in a triangular network, each pair of nodes shares two common neighbors, making local heuristics effective. Conversely, in a hexagonal network, where only length-5 paths connect each node pair, global heuristics may yield better results. Hence, the adaptive integration of multi-range topological information from both local and global heuristics is essential for accurate predictions.

While heuristics prove effective in link prediction tasks, they inherently capture specific topology patterns, posing challenges in their generalization to diverse graphs~\cite{NBFNet,Neo-GNN}. Moreover, heuristics are unable to leverage node features, limiting their efficacy on attributed graphs~\cite{SEAL}.
To make heuristics more universal and general, recent research efforts have been directed toward establishing formulations for heuristics and learning heuristics from these formulations. Notable examples include SEAL~\cite{SEAL}, NBFNet~\cite{NBFNet}, and Neo-GNN~\cite{Neo-GNN}. SEAL's $\gamma$-decaying framework and NBFNet's path formulation are tailored for global heuristics, while Neo-GNN's MLP framework is tailored for local ones. 
To obtain multi-range topological information, a unified formulation that accommodates both local and global heuristics is necessary, yet it remains undiscovered.

Our motivation for constructing the unified formulation stems from the observation that both local and global heuristics can be expressed through adjacency matrix multiplications. Therefore, we unify local and global heuristics into a matrix formulation, enabling the accommodation and generalization of various local and global heuristics. 
In contrast to previous works that construct formulations based on abstract functions such as SEAL~\cite{SEAL}, NBFNet~\cite{NBFNet}, and Neo-GNN~\cite{Neo-GNN}, our unified formulation is developed through direct matrix operations. This unified formulation ensures rigorous equivalence to numerous local and global heuristics under specific configurations.

To learn generalized heuristics and acquire multi-range information, we propose the \textbf{H}euristic \textbf{L}earning \textbf{G}raph \textbf{N}eural \textbf{N}etwork (HL-GNN) to efficiently implement the formulation.
HL-GNN incorporates intra-layer propagation and inter-layer connections while excluding transformation and activation functions. This enables HL-GNN to effectively reach a depth of around 20 layers, while only requiring the training of a global GNN with a time complexity even lower than GCN.
The adaptive weights in HL-GNN facilitate the integration of multi-range topological information, and govern the trade-off between node features and topological information.

Our comprehensive experiments, conducted on the Planetoid, Amazon, and OGB datasets, confirm the effectiveness and efficiency of our proposed HL-GNN. It consistently achieves state-of-the-art performance across numerous benchmarks, maintains excellent scalability, and stands out as the most parameter-efficient method among existing GNN methods. Furthermore, it demonstrates superior speed, surpassing existing heuristic-inspired methods by several orders of magnitude. 
HL-GNN is highly interpretable, as evidenced by the generalized heuristics and learned weights on real-world datasets as well as synthetic datasets.

Our contributions can be summarized as follows:

\begin{itemize}
\item We unify local and global heuristics into a matrix formulation, facilitating the accommodation and generalization of heuristics. We demonstrate that numerous traditional heuristics align with our formulation under specific configurations.

\item We propose HL-GNN to efficiently implement the formulation, capable of reaching a depth of around 20 layers with lower time complexity than GCN. HL-GNN can adaptively balance the trade-off between node features and topological information.

\item Comprehensive experiments demonstrate that HL-GNN outperforms existing methods in terms of performance and efficiency. The interpretability of HL-GNN is highlighted through the analysis of generalized heuristics and learned weights.

\end{itemize}

\section{Related Works}
\subsection{Graph Neural Networks}
Graph Neural Networks (GNNs) have emerged as a powerful paradigm for learning node representations by exploiting neural networks to manipulate both node features and graph topology. These networks employ a message-passing mechanism, with notable examples including Graph Convolutional Network (GCN)~\cite{GCN}, GraphSAGE~\cite{graphsage}, and Graph Attention Network (GAT)~\cite{GAT}. 
Through iterative message propagation, GNNs enable each node representation to accumulate information from its neighboring nodes, thereby facilitating downstream tasks.

While GNNs have emerged as potent solutions for node and graph classification~\cite{dropedge,nodeformer,li2019semi,pooling}, they sometimes fall short in link prediction scenarios compared to traditional heuristics like Common Neighbors (CN)~\cite{CN} or the Resource Allocation Index (RA)~\cite{RA}. The primary issue lies in the inherent intertwining of node features and graph topology during the message-passing process in conventional GNNs. This entanglement causes node features to interfere with graph topology, impeding the effective extraction of topological information for link prediction tasks.

Although in principle an arbitrary number of GNN layers can be stacked, practical GNNs are usually shallow, typically consisting of 2-3 layers, as conventional GNNs often experience a sharp performance drop after just 2 or 3 layers. A widely accepted explanation for this performance degradation with increasing depth is the over-smoothing issue~\cite{over-smoothing,JKNet}, which refers to node representations becoming non-discriminative when going deep. While the adaptive integration of both local and global topological information is essential for link prediction, conventional GNNs usually cannot penetrate beyond 3 layers, restricting the extraction of global topological information.

\subsection{Link Prediction}
Link prediction predicts the likelihood of a link forming between two nodes in a graph.
The problem of link prediction has traditionally been addressed by heuristic methods. These methods are primarily concerned with quantifying the similarity between two nodes based on the graph topology. Heuristic methods can be broadly categorized into two groups: local and global~\cite{survey1,survey2}.

Local heuristics can be further divided into entirety-based heuristics and individual-based heuristics. Entirety-based heuristics, like Common Neighbors (CN)~\cite{CN} and the Local Leicht-Holme-Newman Index (LLHN)~\cite{LHN}, consider the cumulative count of common neighbors. In contrast, individual-based heuristics, exemplified by the Resource Allocation Index (RA)~\cite{RA}, focus on nodes within the common neighborhood and incorporate detailed topological information such as the degree of each node.

Global heuristics, on the other hand, leverage the entire graph topology. Methods such as the Katz Index (KI)~\cite{KI} and the Global Leicht-Holme-Newman Index (GLHN)~\cite{LHN} consider all possible paths between node pairs. The Random Walk with Restart (RWR)~\cite{PPR} assesses the similarity between two nodes based on random walk probabilities. Some global heuristics are tailored to specific path lengths, like the Local Path Index (LPI)~\cite{LP} and the Local Random Walks (LRW)~\cite{LRW}.

Traditional heuristic methods are manually designed and show limitations on complex real-world graphs, prompting a shift toward learning-based approaches. Embedding methods, including Matrix Factorization~\cite{matrix_factorization}, DeepWalk~\cite{deepwalk}, LINE~\cite{line}, and Node2vec~\cite{node2vec}, factorize network representations into low-dimensional node embeddings. However, embedding methods face limitations due to their inability to leverage node features on attributed graphs. 

Recent advancements have focused on enhancing GNNs with valuable topological information. Subgraph GNNs like SEAL~\cite{SEAL}, GraIL~\cite{grail}, and SUREL~\cite{SUREL} explicitly encode subgraphs around node pairs. However, they require the running of a subgraph GNN with the labeling trick for each link during training and inference.
Taking a different perspective, models like NBFNet~\cite{NBFNet} and RED-GNN~\cite{redgnn} adopt source-specific message passing, drawing inspiration from global heuristics. However, they require training a global GNN for each source node.
Some methods opt for a single global GNN to improve scalability and efficiency. Neo-GNN~\cite{Neo-GNN} uses two MLPs, while SEG~\cite{SEG} uses a GCN layer and an MLP to approximate a heuristic function. BUDDY~\cite{BUDDY} develops a novel GNN that passes subgraph sketches as messages. However, these methods primarily focus on local topological information and struggle to capture global topological information.
In contrast, the proposed HL-GNN can capture long-range information up to 20 hops while only requiring the training of a global GNN. Further details about the comparison of HL-GNN with existing methods are provided in Section~\ref{para:comparison}.

\section{Unified Heuristic Formulation}

Let $\mathcal{G} = (\mathcal{V}, \mathcal{E})$ denote a graph, with nodes $\mathcal{V}$ and edges $\mathcal{E}$. In this work, we consider undirected and unweighted graphs. We define $|\mathcal{V}|=N$ as the number of nodes and $|\mathcal{E}|=M$ as the number of edges. 
Node features are characterized by the node feature matrix $\boldsymbol{X} \in \mathbb{R}^{N \times F}$, where $F$ indicates the number of features. 
The graph topology is encapsulated by the adjacency matrix $\boldsymbol{A} \in {\{0,1\}}^{N \times N}$. The matrix $\tilde{\boldsymbol{A}}=\boldsymbol{A}+\boldsymbol{I}_N$ represents the adjacency matrix with self-loops, where $\tilde{a}_{ij}=1$ signifies an edge between nodes $i$ and $j$. The node degree $i$ with self-loops is given by $\tilde{d}_i=\sum_j \tilde{a}_{ij}$, with the diagonal degree matrix with self-loops denoted as $\tilde{\boldsymbol{D}}=\operatorname{diag}(\tilde{d}_1, \cdots, \tilde{d}_N)$. The set $\Gamma_x$ represents the 1-hop neighbors of node $x$, encompassing node $x$ itself.

We introduce a set of normalized adjacency matrices, detailed in Table~\ref{tab:adjacency_matrices}. This set comprises the symmetrically normalized matrix $\Asym$, the row-stochastic normalized matrix $\Ars$, and the column-stochastic normalized matrix $\Acs$, which encompass diverse normalization techniques (left multiplication, right multiplication, or both) applied to the adjacency matrix.
Next, we define the propagation operator $\mathbb{A}$ to offer a choice among different types of adjacency matrices:
\begin{definition}
    (\textit{Propagation operator}). The \textit{propagation operator} $\mathbb{A} \in \mathbb{R}^{N \times N}$ is defined as $\mathbb{A} \in \{\A, \Asym, \Ars, \Acs\}$. The expressions for the adjacency matrices $\A, \Asym, \Ars, \Acs$ are detailed in Table~\ref{tab:adjacency_matrices}.
\end{definition}

The propagation operator encapsulates the prevalent propagation mechanisms commonly employed in GNNs.
By substituting the adjacency matrix $\A$ with the propagation operator $\mA$, we can opt for various propagation mechanisms that deliver diverse information.

\begin{table}[tbp]
    \small
    \caption{Notations and expressions of adjacency matrices.}
    \label{tab:adjacency_matrices}
    \centering
    \vspace{-5px}
    \setlength{\tabcolsep}{8pt}
    \begin{tabular}{lcc} 
         \toprule
         \textbf{Adjacency Matrix} & \textbf{Notation} & \textbf{Expression} \\
         \midrule
         Matrix with Self-Loops & $\tilde{\boldsymbol{A}}$ & $\boldsymbol{A} + \boldsymbol{I}_N$ \\
         Symmetrical Matrix & $\Asym$ & $\tilde{\boldsymbol{D}}^{-1/2} \tilde{\boldsymbol{A}} \tilde{\boldsymbol{D}}^{-1/2}$ \\
         Row-Stochastic Matrix & $\Ars$ & $\tilde{\boldsymbol{D}}^{-1} \tilde{\boldsymbol{A}}$ \\
         Column-Stochastic Matrix & $\Acs$ & $\tilde{\boldsymbol{A}} \tilde{\boldsymbol{D}}^{-1}$ \\
         Propagation Operator & $\mathbb{A}$ & $\{\A, \Asym, \Ars, \Acs\}$ \\
         \bottomrule
    \end{tabular}
\end{table}



\begin{table*}[tbp]
\caption{A selection of traditional local and global heuristics, their mathematical expressions, their matrix forms, and specific configurations within the unified heuristic formulation for alignment.}
\label{tab:heursitics}
\centering
\small
\begin{threeparttable}
\begin{tabular}{llcccc}
    \toprule
    \textbf{Type} & \textbf{Method} & \textbf{Expression}  &\textbf{Matrix Form}& \textbf{Propagation Operators $\mA^{(m)}$} & \textbf{Weight Parameters} $\beta^{(l)}$ \\
    \midrule
    Local & CN & $|\Gamma_i \cap \Gamma_j|$  &$(\A^2)_{i, j}$& $\mA^{(1)}=\mA^{(2)} = \A$ & $\beta^{(0)}= \beta^{(1)} = 0$, $\beta^{(2)} = 1$ \\
    Local & LLHN & $\frac{|\Gamma_i \cap \Gamma_j|}{\tilde{d}_i \tilde{d}_j}$  &$(\Ars \Acs)_{i, j}$& \(\mA^{(1)} = \Ars\), \(\mA^{(2)} = \Acs\) & \(\beta^{(0)}=\beta^{(1)} = 0\), \(\beta^{(2)} = 1\) \\
    Local & RA & $\sum_{k \in \Gamma_i \cap \Gamma_j} \frac{1}{\tilde{d}_k}$  &$(\Acs \A)_{i, j}$& \(\mA^{(1)} = \Acs\), \(\mA^{(2)} = \A\) \tnote{*}  & \(\beta^{(0)}=\beta^{(1)} = 0\), \(\beta^{(2)} = 1\) \\
    Global& KI & $\sum_{l=1}^{\infty} \gamma^l |{\textnormal{paths}}_{i,j}^l|$  &$\left(\sum_{l=1}^{\infty} \gamma^l\A^l\right)_{i,j}$ & $\mA^{(m)} = \A$ for $m \geq 1$ & $\beta^{(0)} = 0$, $ \beta^{(l)}=\gamma^l$ for $l \geq 1$ \\
    Global& GLHN & $\sum_{l=0}^{\infty} \phi^l |{\textnormal{paths}}_{i,j}^l|$  &$\left(\boldsymbol{I}_N+\sum_{l=1}^{\infty} \phi^l \A^l \right)_{i,j}$& $\mA^{(m)} = \A$ for $m \geq 1$ & $\beta^{(0)} = 1$, $ \beta^{(l)}=\phi^l$ for $l \geq 1$ \\
    Global& RWR & $[\boldsymbol{\pi}_i(\infty)]_j$  &$\left(\sum_{l=0}^{\infty}(1-\alpha)\alpha^l \Ars^l\right)_{i,j}$& $\mA^{(m)} = \Ars$ for $m \geq 1$ & $\beta^{(l)}=(1-\alpha)\alpha^l$ for $l \geq 0$ \\
    Global & LPI & $\sum_{l=2}^L \gamma^{l-2} |{\textnormal{paths}}_{i,j}^l|$  &$\left(\sum_{l=2}^L \gamma^{l-2} \A^l\right)_{i,j}$& $\mA^{(m)} = \A$ for $1 \leq m \leq L$ & $\beta^{(0)} = \beta^{(1)} =0$, $ \beta^{(l)}=\gamma^{l-2}$ for $2 \leq l \leq L$ \\
    Global & LRW & $\frac{\tilde{d}_i}{2M}[\boldsymbol{\pi}_i(L)]_j$  &$\left(\sum_{l=0}^{L-1}\frac{\tilde{d}_i}{2M}(1-\alpha)\alpha^l \Ars^l\right)_{i,j}$& $\mA^{(m)} = \Ars$ for $1 \leq m \leq L-1$ & $\beta^{(l)}=\frac{\tilde{d}_i}{2M} (1-\alpha)\alpha^l$ for $0 \leq l \leq L-1$ \\
    \bottomrule
\end{tabular}
\begin{tablenotes}
    \footnotesize
    \item[*] When setting \(\mA^{(1)} = \A\), \(\mA^{(2)} = \Ars\) in the formulation, it also aligns with the RA Index.
  \end{tablenotes}
\end{threeparttable}
\end{table*}

Given that heuristics are fundamentally influenced by graph topology, it is possible to express various heuristics using adjacency matrices. 
The $(i,j)$ entry of the 2-order adjacency matrix multiplication denotes the count of common neighbors for nodes $i$ and $j$. The $(i,j)$ entry of the $l$-order adjacency matrix multiplications denotes the number of length-$l$ paths between nodes $i$ and $j$. Hence, by employing distinct orders of adjacency matrix multiplications, we can extract varying insights from neighbors or paths.
Following this intuition, we can express diverse heuristics in matrix form.
We provide a concise summary of heuristics, their mathematical expressions, and their corresponding matrix forms in Table~\ref{tab:heursitics}.
Detailed derivations of matrix forms of heuristics can be found in Appendix~\ref{para:heuristic_proofs}.  
Next, we introduce the definition of the heuristic formulation:
\begin{definition}
  (\textit{Heuristic formulation}). A \textit{heuristic formulation} is denoted by a matrix $\boldsymbol{H} \in \mathbb{R}^{N \times N}$. Each entry $(i, j)$ in this matrix corresponds to the heuristic score for the link $(i, j)$, denoted as $\boldsymbol{H}_{i, j}=h(i, j)$.
\end{definition}

We can unify both local and global heuristics in a formulation based on matrix forms of heuristics. Our proposed heuristic formulation parameterizes a combination of matrix multiplications:
\begin{equation}
\label{eq:formulation}
\boldsymbol{H} 
= \sum_{l=0}^{L} \left(\beta^{(l)} \prod_{m=0}^l \mathbb{A}^{(m)} \right),
\end{equation}
where $\mathbb{A}^{(m)} \in \{\A, \Asym, \Ars, \Acs\}$ for $1 \leq m \leq L$ represent the propagation operators, and $\mathbb{A}^{(0)} = \boldsymbol{I}_N$. The coefficients $\beta^{(l)}$ for $0 \leq l \leq L$ modulate the weights of different orders of matrix multiplications, and $L$ is the maximum order. 
Numerous traditional heuristics align with our formulation under specific configurations.
Table~\ref{tab:heursitics} showcases a selection of traditional heuristics and illustrates their alignment with our formulation through propagation operators $\mathbb{A}^{(m)}$ and weights $\beta^{(l)}$. 
We assert the formulation's ability to accommodate heuristics in Proposition~\ref{prop:formulation-heuristic}. The proof for Proposition~\ref{prop:formulation-heuristic} can be found in Appendix~\ref{para:heuristic_proofs}.

\begin{proposition}
\label{prop:formulation-heuristic}
    Our formulation can accommodate a broad spectrum of local and global heuristics with propagation operators $\mathbb{A}^{(m)}$ for $1 \leq m \leq L$, weight parameters $\beta^{(l)}$ for $0 \leq l \leq L$, and maximum order $L$. 
\end{proposition}

Unlike previous methods that exclusively cater to either local or global heuristics, our formulation seamlessly integrates both aspects, presenting a unified solution. In contrast to prior works relying on abstract functions for heuristic approximation~\cite{SEAL, NBFNet, Neo-GNN, SEG}, our formulation is developed through direct matrix operations. This formulation facilitates rigorous equivalence to numerous  local and global heuristics under specific configurations.
It is crucial to note that our heuristic formulation does not aim to accommodate all possible heuristics. 
Instead, it aims to distill the critical characteristics of heuristics, 
with a specific focus on extracting common neighbors from local heuristics and global paths from global heuristics.

Existing heuristics are primarily handcrafted and may not be optimal for real-world graphs.
Leveraging the propagation operators, weight parameters and maximum order offers the potential to learn generalized, possibly more effective heuristics, which we will discuss in Section~\ref{para:interpret} and Appendix~\ref{para:heuristic_advance}.

\section{Heuristic Learning Graph Neural Network (HL-GNN)}

\subsection{Heuristic Learning Graph Neural Network}
\subsubsection{Motivation}
Direct matrix multiplication serves as a straightforward method to implement the heuristic formulation in Equation~\eqref{eq:formulation}.
However, it comes with high computational and memory costs. 
The time complexity of direct matrix multiplication is $\mathcal{O}(L^2 N^3)$ and the space complexity is $\mathcal{O}(LN^2)$, where $N$ denotes the number of nodes. This is attributed to executing up to $L$-order matrix multiplications for $L$ times.
The significant time and space complexities present two major challenges of ensuring scalability and maintaining depth:
\begin{itemize}[leftmargin=*]
\item \textit{Scalability.} To be scalable, the model must effectively handle large graphs.
Datasets like OGB are substantially larger than those like Planetoid, making the value of $N$ a considerable strain on the time and space complexities.
\item  \textit{Depth.} To effectively integrate global heuristics into the heuristic formulation, the value of $L$ must be sufficiently large to encapsulate messages from distant nodes. However, increasing $L$ further strains the time and space complexities.
\end{itemize}
Consequently, there is a pressing need to mitigate the burdens of both time and space complexities.

\subsubsection{Architecture}

The construction and computation of $N \times N$ matrices impose significant computational and memory demands.
One potential technique is initializing $\mathbb{A}^{(0)} = \boldsymbol{X}$ instead of $\mathbb{A}^{(0)} = \boldsymbol{I}_N$.
This approach effectively reduces the feature dimensionality from $N$ to $F$, resulting in substantial time and space savings. Moreover, since heuristics cannot leverage node features on attributed graphs, initializing with $\mathbb{A}^{(0)} = \boldsymbol{X}$ allows the heuristic formulation to utilize node features.
Even if node features are of low quality or completely absent, we can still train embeddings for each node. Therefore, $\boldsymbol{X}$ can represent either raw node features or learnable node embeddings.

Further, we exploit the sparsity of the graph and employ a Graph Neural Network to compute the heuristic formulation. We propose the efficient and scalable \textbf{H}euristic \textbf{L}earning \textbf{G}raph \textbf{N}eural \textbf{N}etwork (HL-GNN), expressed as:
\begin{equation}
\label{eq:HLGNN}
\boldsymbol{Z}^{(0)}=\boldsymbol{X},
\quad 
\boldsymbol{Z}^{(l)}=\mA^{(l)} \boldsymbol{Z}^{(l-1)},  
\quad 
\boldsymbol{Z} = \sum_{l=0}^L \beta^{(l)} \boldsymbol{Z}^{(l)},
\end{equation}
with $\beta^{(l)}$ representing the learnable weight of the $l$-th layer, and $L$ representing the model's depth.
An illustration of our proposed HL-GNN is provided in Figure~\ref{fig:HL-GNN}.
We do not impose constraints on $\beta^{(l)}$, allowing them to take positive or negative values.
Adaptive weights $\beta^{(l)}$ facilitate the integration of multi-range topological information and govern the trade-off between node features and topological information.
Given the discrete nature of the propagation operators $\mA^{(l)} \in \{\A, \Asym, \Ars, \Acs\}$, 
they obstruct the back-propagation process, necessitating their relaxation to a continuous form:
\begin{equation}
\label{eq:relaxation}
\mA^{(l)}= \alpha_1^{(l)} \Ars+\alpha_2^{(l)} \Acs+\alpha_3^{(l)} \Asym, \quad \text{for} \quad 1 \leq l \leq L,
\end{equation}
where $\alpha_1^{(l)}, \alpha_2^{(l)}, \alpha_3^{(l)}$ are layer-specific learnable weights harmonizing three propagation mechanisms. 
The continuous relaxation of the propagation operators enables gradient back-propagation, thereby allowing the model to be trained end-to-end.
We exclude the adjacency matrix $\A$ in Equation~\eqref{eq:relaxation} to ensure that the eigenvalues of $\mA^{(l)}$ fall within the range $[0, 1]$. Moreover, we apply a softmax function to $\boldsymbol{\alpha}^{(l)}$, where $\operatorname{softmax}(\alpha_i^{(l)})=\exp(\alpha_i^{(l)}) /\sum_{j=1}^3 \exp(\alpha_j^{(l)})$ for $i=1, 2, 3$.
Controlling the eigenvalues of $\mA^{(l)}$ helps prevent numerical instabilities as well as issues related to exploding gradients or vanishing gradients.

HL-GNN employs intra-layer propagation and inter-layer connections as described in Equation~\eqref{eq:HLGNN}. 
The salient trait is its elimination of representation transformation and non-linear activation at each layer, requiring only a few trainable parameters.
We assert the relationship between the learned representations $\boldsymbol{Z}$ and the heuristic formulation $\boldsymbol{H}$ in Proposition~\ref{prop:relationship}. The proof for Proposition~\ref{prop:relationship} can
be found in Appendix~\ref{para:relationship_proof}.

\begin{proposition}
\label{prop:relationship}
    The relationship between the learned representations $\boldsymbol{Z}$ in Equation~\eqref{eq:HLGNN} and the heuristic formulation $\boldsymbol{H}$ in Equation~\eqref{eq:formulation} is given by $\boldsymbol{Z} = \boldsymbol{H}\boldsymbol{X}$, where $\boldsymbol{X}$ is the node feature matrix.
\end{proposition}

According to Proposition~\ref{prop:relationship}, the learned representations $\boldsymbol{Z}$ utilize heuristics as weights to combine features from all nodes. 
The heuristic formulation $\boldsymbol{H}$ can be effectively distilled through the message-passing process in HL-GNN. Consequently, HL-GNN has the ability to accommodate and generalize both local and global heuristics.
Our method can be viewed as topological augmentation, employing the topological information embedded in $\boldsymbol{H}$ to enhance raw node features $\boldsymbol{X}$.

For sparse graphs, the time complexity of HL-GNN is $\mathcal{O}(LMF)$, where $M$ is the number of edges. The space complexity of HL-GNN is $\mathcal{O}(NF)$. On a large graph, typically containing millions of nodes, HL-GNN leads to remarkable time and space savings -- ten and five orders of magnitude, respectively -- compared to direct matrix multiplication.

\begin{figure}[tbp]
    \centering
    \includegraphics[width=0.47\textwidth]{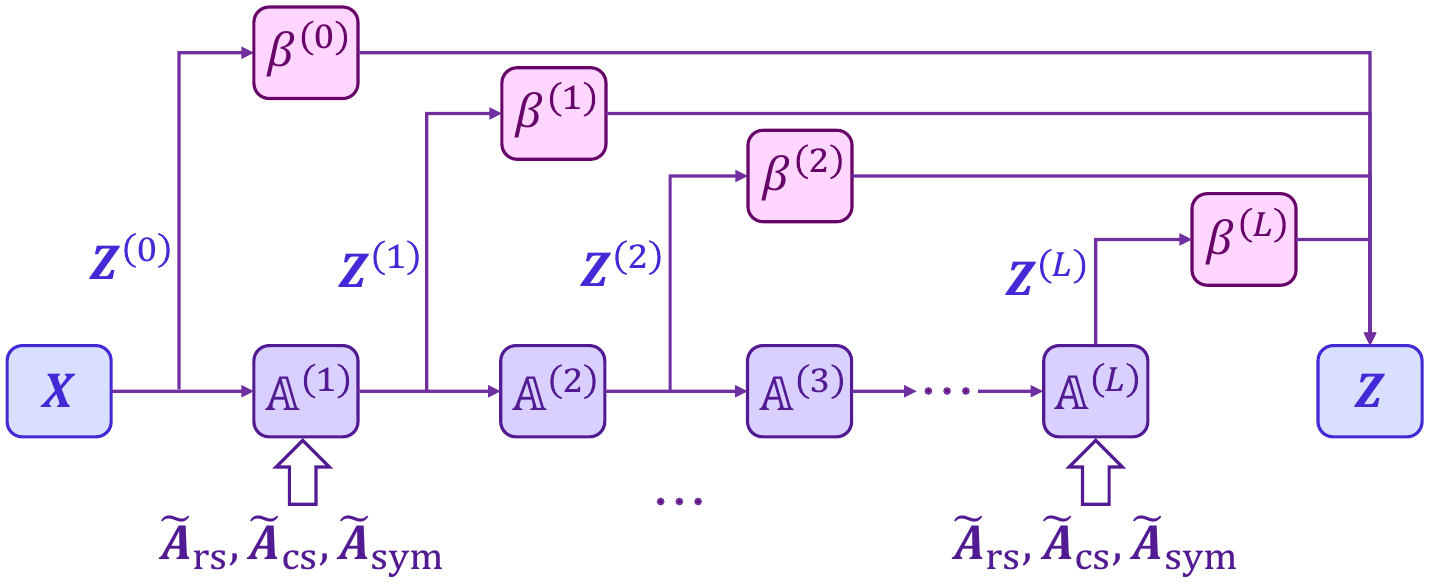}
    \caption{Illustration of the proposed Heuristic Learning Graph Neural Network (HL-GNN). Every rounded rectangle symbolizes a left multiplication operation.}
    \label{fig:HL-GNN}
\end{figure}

\subsubsection{Training}
After acquiring the node representations, 
we employ a predictor to compute the likelihood for each link by 
$s_{i j} = f_\theta(\boldsymbol{z}_i \odot \boldsymbol{z}_j)$,
where $f_\theta$ is a feed-forward neural network, $\boldsymbol{z}_i$ and $\boldsymbol{z}_j$ represent the representations of node $i$ and $j$ respectively, 
and the symbol \(\odot\) denotes the element-wise product.

Many methods categorize link prediction as a binary classification problem and conventionally employ the cross-entropy loss function. However, this might not always be the suitable strategy. Standard evaluation procedures in link prediction do not label positive pairs as 1 and negative pairs as 0. The primary objective is to rank positive pairs higher than negative pairs, aligning with the maximization of the Area Under the Curve (AUC). 
In light of this, we adopt the AUC loss as described in~\cite{AUC_loss}, ensuring it aligns conceptually with the evaluation procedure:
\begin{equation}
\mathcal{L} = \min_{\alpha, \beta, \theta} 
\sum_{ (i, j) \in \mathcal{E} } \sum_{ (i, k) \in \mathcal{E}^{-} } 
\gamma_{i j} \left(\max(0, \gamma_{i j} - s_{ij} + s_{ik})\right)^2.
\end{equation}
Here, \(\mathcal{E}^{-}\) signifies the negative links uniformly sampling from the set $\mathcal{V} \times \mathcal{V}-\mathcal{E}$, 
and \(\gamma_{ij}\) is an adaptive margin between positive link $(i, j)$ and negative link $(i, k)$. 
The model is trained end-to-end, jointly optimizing the GNN parameters $\alpha$ and $\beta$, along with the predictor parameters $\theta$.

\subsection{Comparison with Existing Methods}
\label{para:comparison} 
We evaluate the heuristic-learning ability, information range, and time complexity of HL-GNN by comparing it with conventional GNNs and heuristic-inspired GNN methods. A summary of these comparisons is provided in Table~\ref{tab:comparison}.
HL-GNN excels at accommodating and generalizing a wide range of both local and global heuristics.
In contrast, SEAL~\cite{SEAL} focuses on subgraphs to learn local heuristics, while NBFNet~\cite{NBFNet} concentrates on paths to learn global heuristics. 
Neo-GNN~\cite{Neo-GNN} leverages two MLPs for local heuristic learning, and BUDDY~\cite{BUDDY} uses subgraph sketches to represent local heuristics. Notably, most of these methods are limited to topological information within a 3-hop range. In contrast, HL-GNN can reach a depth of approximately 20 layers, providing a broader information range. Adaptive weights in HL-GNN enable the integration of both local and global topological information.

HL-GNN has a time complexity of $\mathcal{O}(LMF)$, which is the lowest among the compared methods. Unlike conventional GNNs, HL-GNN solely utilizes propagation mechanisms and omits transformation and activation functions.
SEAL requires running a subgraph GNN with the labeling trick for each link, and NBFNet requires running a global GNN for each source node during training and inference. In contrast, HL-GNN only requires running a single global GNN during training and inference. Furthermore, HL-GNN avoids the need to extract topological information from common neighbors and subgraph sketches, as required by Neo-GNN and BUDDY, respectively. A detailed time complexity analysis of each method is included in Appendix~\ref{para:time_comp_analysis}.

\begin{table}[tbp]
    \caption{Comparison of heuristic-learning ability, information range, and time complexity of HL-GNN with conventional GNNs and heuristic-inspired GNN methods.}
    \label{tab:comparison}
    \centering
    \small
    \begin{tabular}{lccc}
    \toprule
         & \textbf{Learned Heuristics}& \textbf{Range}&\textbf{Time Complexity}\\
    \midrule
         \textbf{GCN}&N.A.&3 hops&$\mathcal{O}(LF(M+NF))$\\
         \textbf{GAT}&N.A.&3 hops&$\mathcal{O}(LKND^2F^2)$\\
         \textbf{SEAL}&Local&3 hops&$\mathcal{O}(M(V^2 + LEF))$\\
         \textbf{NBFNet}&Global&6 hops&$\mathcal{O}(LNF(M+NF))$\\
         \textbf{Neo-GNN}&Local&2 hops&$\mathcal{O}(LMF+NDF^2)$\\
         \textbf{BUDDY}& Local&3 hops&$\mathcal{O}(LM(LH+F^2))$\\
    \midrule
         \textbf{HL-GNN}&Local / Global&20 hops&$\mathcal{O}(LMF)$\\
    \bottomrule
    \end{tabular}
\end{table}

\section{Experiments}
\begin{table*}[tbp]
    \caption{Results on link prediction benchmarks including the Planetoid, Amazon, and OGB datasets. Results are presented as average $\pm$ standard deviation. The best and second-best performances are marked with \textbf{bold} and \underline{underline}, respectively. OOM denotes out of GPU memory.}
    \label{tab:main_result}
    \centering
    \begin{tabular}{lccccccccc}
    \toprule
    & \texttt{Cora} & \texttt{Citeseer} & \texttt{Pubmed} & \texttt{Photo} & \texttt{Computers} & \texttt{collab} & \texttt{ddi} & \texttt{ppa} & \texttt{citation2} \\
    & Hits@100 & Hits@100 & Hits@100 & AUC & AUC & Hits@50 & Hits@20 & Hits@100 & MRR \\
    \midrule
    \textbf{CN} & 33.92\footnotesize{$\pm$0.46} & 29.79\footnotesize{$\pm$0.90} & 23.13\footnotesize{$\pm$0.15} & 96.73\footnotesize{$\pm$0.00} & 96.15\footnotesize{$\pm$0.00} & 56.44\footnotesize{$\pm$0.00} & 17.73\footnotesize{$\pm$0.00} & 27.65\footnotesize{$\pm$0.00} & 51.47\footnotesize{$\pm$0.00} \\
    \textbf{RA} & 41.07\footnotesize{$\pm$0.48} & 33.56\footnotesize{$\pm$0.17} & 27.03\footnotesize{$\pm$0.35} & 97.20\footnotesize{$\pm$0.00} & 96.82\footnotesize{$\pm$0.00} & 64.00\footnotesize{$\pm$0.00} & 27.60\footnotesize{$\pm$0.00} & 49.33\footnotesize{$\pm$0.00} & 51.98\footnotesize{$\pm$0.00} \\
    \textbf{KI} & 42.34\footnotesize{$\pm$0.39} & 35.62\footnotesize{$\pm$0.33} & 30.91\footnotesize{$\pm$0.69} & 97.45\footnotesize{$\pm$0.00} & 97.05\footnotesize{$\pm$0.00} & 59.79\footnotesize{$\pm$0.00} & 21.23\footnotesize{$\pm$0.00} & 24.31\footnotesize{$\pm$0.00} & 47.83\footnotesize{$\pm$0.00} \\
    \textbf{RWR} & 42.57\footnotesize{$\pm$0.56} & 36.78\footnotesize{$\pm$0.58} & 29.77\footnotesize{$\pm$0.45} & 97.51\footnotesize{$\pm$0.00} & 96.98\footnotesize{$\pm$0.00} & 60.06\footnotesize{$\pm$0.00} & 22.01\footnotesize{$\pm$0.00} & 22.16\footnotesize{$\pm$0.00} & 45.76\footnotesize{$\pm$0.00} \\
    \midrule
    \textbf{MF} & 64.67\footnotesize{$\pm$1.43} & 65.19\footnotesize{$\pm$1.47} & 46.94\footnotesize{$\pm$1.27} & 97.92\footnotesize{$\pm$0.37} & 97.56\footnotesize{$\pm$0.66} & 38.86\footnotesize{$\pm$0.29} & 13.68\footnotesize{$\pm$4.75} & 32.29\footnotesize{$\pm$0.94} & 51.86\footnotesize{$\pm$4.43} \\
    \textbf{Node2vec} & 68.43\footnotesize{$\pm$2.65} & 69.34\footnotesize{$\pm$3.04} & 51.88\footnotesize{$\pm$1.55} & 98.37\footnotesize{$\pm$0.33} & 98.21\footnotesize{$\pm$0.39} & 48.88\footnotesize{$\pm$0.54} & 23.26\footnotesize{$\pm$2.09} & 22.26\footnotesize{$\pm$0.88} & 61.41\footnotesize{$\pm$0.11} \\
    \textbf{DeepWalk} & 70.34\footnotesize{$\pm$2.96} & 72.05\footnotesize{$\pm$2.56} & 54.91\footnotesize{$\pm$1.25} & 98.83\footnotesize{$\pm$0.23} & 98.45\footnotesize{$\pm$0.45} & 50.37\footnotesize{$\pm$0.34} & 26.42\footnotesize{$\pm$6.10} & 35.12\footnotesize{$\pm$0.79} & 55.58\footnotesize{$\pm$1.75} \\
    \midrule
    \textbf{GCN} & 66.79\footnotesize{$\pm$1.65} & 67.08\footnotesize{$\pm$2.94} & 53.02\footnotesize{$\pm$1.39} & 98.61\footnotesize{$\pm$0.15} & 98.55\footnotesize{$\pm$0.27} & 47.14\footnotesize{$\pm$1.45} & 37.07\footnotesize{$\pm$5.07} & 18.67\footnotesize{$\pm$1.32} & 84.74\footnotesize{$\pm$0.21} \\
    \textbf{GAT} & 60.78\footnotesize{$\pm$3.17} & 62.94\footnotesize{$\pm$2.45} & 46.29\footnotesize{$\pm$1.73} & 98.42\footnotesize{$\pm$0.19} & 98.47\footnotesize{$\pm$0.32} & 55.78\footnotesize{$\pm$1.39} & 54.12\footnotesize{$\pm$5.43} & 19.94\footnotesize{$\pm$1.69} & 86.33\footnotesize{$\pm$0.54} \\
    \textbf{SEAL} & 81.71\footnotesize{$\pm$1.30}& 83.89\footnotesize{$\pm$2.15}& \underline{75.54\footnotesize{$\pm$1.32}} & 98.85\footnotesize{$\pm$0.04}& \underline{98.70\footnotesize{$\pm$0.18}} & 64.74\footnotesize{$\pm$0.43}& 30.56\footnotesize{$\pm$3.86} & 48.80\footnotesize{$\pm$3.16} & \underline{87.67\footnotesize{$\pm$0.32}}\\
    \textbf{NBFNet} & 71.65\footnotesize{$\pm$2.27} & 74.07\footnotesize{$\pm$1.75} & 58.73\footnotesize{$\pm$1.99} & 98.29\footnotesize{$\pm$0.35} & 98.03\footnotesize{$\pm$0.54} & OOM & 4.00\footnotesize{$\pm$0.58} & OOM & OOM \\
    \textbf{Neo-GNN} & 80.42\footnotesize{$\pm$1.31} & 84.67\footnotesize{$\pm$2.16}& 73.93\footnotesize{$\pm$1.19} & 98.74\footnotesize{$\pm$0.55} & 98.27\footnotesize{$\pm$0.79} & 62.13\footnotesize{$\pm$0.58} & 63.57\footnotesize{$\pm$3.52}& 49.13\footnotesize{$\pm$0.60} & 87.26\footnotesize{$\pm$0.84} \\
    \textbf{BUDDY} & \underline{88.00\footnotesize{$\pm$0.44}} & \underline{92.93\footnotesize{$\pm$0.27}} & 74.10\footnotesize{$\pm$0.78} & \underline{99.05\footnotesize{$\pm$0.21}} & 98.69\footnotesize{$\pm$0.34} & \underline{65.94\footnotesize{$\pm$0.58}} & \underline{78.51\footnotesize{$\pm$1.36}} & \underline{49.85\footnotesize{$\pm$0.20}}& 87.56\footnotesize{$\pm$0.11} \\
    \textbf{HL-GNN} & \textbf{94.22\footnotesize{$\pm$1.64}} & \textbf{94.31\footnotesize{$\pm$1.51}} & \textbf{88.15\footnotesize{$\pm$0.38}} & \textbf{99.11\footnotesize{$\pm$0.07}} & \textbf{98.82\footnotesize{$\pm$0.21}} & \textbf{68.11\footnotesize{$\pm$0.54}} & \textbf{80.27\footnotesize{$\pm$3.98}} & \textbf{56.77\footnotesize{$\pm$0.84}}& \textbf{89.43\footnotesize{$\pm$0.83}} \\
    \bottomrule
    \end{tabular}

\end{table*}

\subsection{Experiment Setup}
\label{para:expsetup}

\subsubsection{Datasets}
We utilize nine datasets from three sources: Planetoid~\cite{planetoid}, Amazon~\cite{amazon}, and OGB~\cite{OGB}. 
The Planetoid datasets include \texttt{Cora}, \texttt{Citeseer}, and \texttt{Pubmed}. The Amazon datasets include \texttt{Photo} and \texttt{Computers}. 
The OGB datasets include \texttt{ogbl-collab}, \texttt{ogbl-ddi}, \texttt{ogbl-ppa}, and \texttt{ogbl-citation2}. 
Dataset statistics can be found in Appendix~\ref{para:dataset_statistics}.

\subsubsection{Baselines}
We compare our model against a diverse set of baseline methods, including heuristics like CN~\cite{CN}, RA~\cite{RA}, KI~\cite{KI}, and RWR~\cite{PPR}, traditional embedding-based methods such as MF~\cite{matrix_factorization}, Node2vec~\cite{node2vec}, and DeepWalk~\cite{deepwalk}, as well as conventional GNNs like GCN~\cite{GCN} and GAT~\cite{GAT}. Additionally, we benchmark HL-GNN against heuristic-inspired GNN methods like SEAL~\cite{SEAL}, NBFNet~\cite{NBFNet}, Neo-GNN~\cite{Neo-GNN}, and BUDDY~\cite{BUDDY}. This comprehensive comparison enable us to assess the performance and effectiveness of the proposed HL-GNN.

\subsubsection{Experimental settings}
In accordance with previous works~\cite{NBFNet,BUDDY}, we randomly sample 5\% and 10\% of the links for validation and test sets on non-OGB datasets. We sample the same number of non-edge node pairs as negative links. For the OGB datasets, we follow their official train/validation/test splits.
Following the convention in previous works~\cite{BUDDY,ILP-GNN}, we use Hits@100 as the evaluation metric for the Planetoid datasets, and we use AUC for the Amazon datasets. For the OGB datasets, we use their official evaluation metrics, such as Hits@50 for \texttt{ogbl-collab}, Hits@20 for \texttt{ogbl-ddi}, Hits@100 for \texttt{ogbl-ppa}, and Mean Reciprocal Rank (MRR) for \texttt{ogbl-citation2}~\cite{OGB}. 

We include a linear layer as preprocessing before HL-GNN to align the dimension of node features with the hidden channels of HL-GNN. We also leverage node embeddings on the OGB datasets to enhance the node representations. For the \texttt{ogbl-collab} dataset, we follow OGB's guidelines and use the validation set for training. We evaluate HL-GNN over 10 runs without fixing the random seed. More details about the experiments are provided in Appendix~\ref{para:exp_detail}.

\subsection{Main Results}

As shown in Table~\ref{tab:main_result}, HL-GNN consistently outperforms all the baselines on all of the datasets, highlighting its effectiveness and robustness for link prediction tasks. Table~\ref{tab:main_result} reports the averaged results with standard deviations. Notably, HL-GNN achieves a remarkable gain of 7.0\% and 16.7\% in Hits@100 compared to the second-best method on the Planetoid datasets \texttt{Cora} and \texttt{Pubmed}, respectively.
Moreover, our HL-GNN demonstrates its ability to handle large-scale graphs effectively, as evidenced by its superior performance on the OGB datasets. Specifically, HL-GNN achieves a gain of 13.9\% in Hits@100 on \texttt{ogbl-ppa}, and achieves 68.11\% Hits@50 on \texttt{ogbl-collab} and 89.43\% MRR on \texttt{ogbl-citation2}. Even when node features are absent or of low quality, HL-GNN maintains consistent performance by learning embeddings for each node, as demonstrated on datasets like \texttt{ogbl-ddi}, which lack node features.

HL-GNN outperforms all listed heuristics, indicating its capacity to generalize heuristics and integrate them with node features. According to Table~\ref{tab:main_result}, local heuristics like CN and RA perform better than global heuristics on the OGB datasets, while global heuristics like KI and RWR perform better on the Planetoid and Amazon datasets. This underscores the importance of establishing a unified formulation that accommodates both local and global heuristics. Notably, we can use the configuration in Table~\ref{tab:heursitics} to recover the heuristic RA from HL-GNN without training, achieving a performance of 49.33\% Hits@100 on \texttt{ogbl-ppa}. This result serves as a compelling lower bound for HL-GNN's performance on \texttt{ogbl-ppa}.
HL-GNN significantly outperforms conventional GNNs like GCN and GAT across all datasets. Additionally, HL-GNN also surpasses existing heuristic-inspired GNN methods, including SEAL, NBFNet, Neo-GNN, and BUDDY, suggesting that integrating information from multiple ranges is beneficial for link prediction tasks.

\subsection{Ablation Studies}
\subsubsection{Different information ranges}
The adaptive weights $\beta^{(l)}$ in HL-GNN facilitate the integration of multi-range information, encompassing both local and global topological information. To investigate the impact of information ranges, we conduct experiments isolating either local or global information.
We train a GNN variant using skip-connections of the first 3 layers as the output, exclusively considering local topological information. Similarly, we train another GNN variant using the final-layer output with GNN depth $L \geq 5$ to exclusively consider global topological information.
Figure~\ref{fig:output} demonstrates that HL-GNN consistently outperforms GNN variants focusing solely on local or global topological information. This underscores HL-GNN's efficacy in adaptively combining both types of information.

\begin{figure}[tbp]
    \centering
    \begin{subfigure}{0.49\linewidth}
        \includegraphics[width=\textwidth]{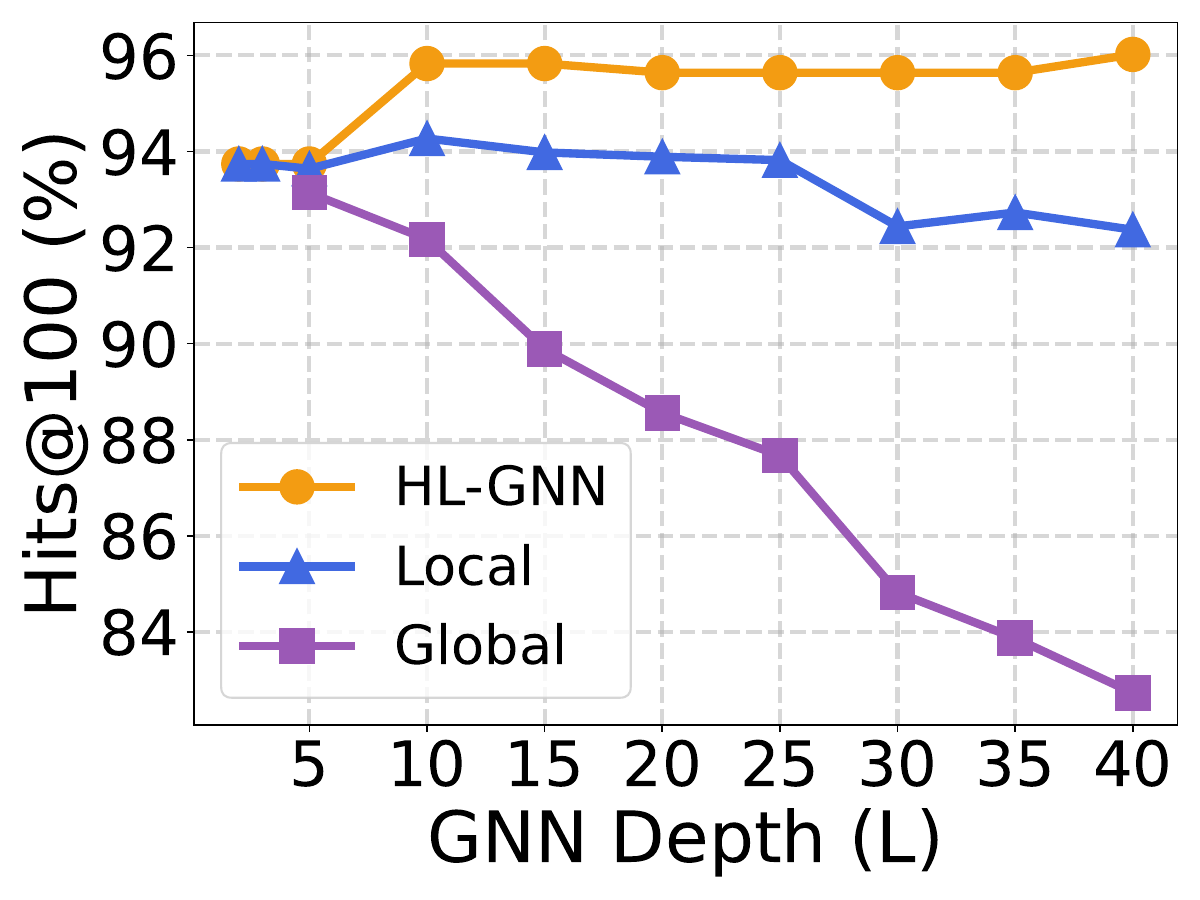}
        \caption{\texttt{Cora}.}
    \end{subfigure}
    \begin{subfigure}{0.49\linewidth}
        \includegraphics[width=\textwidth]{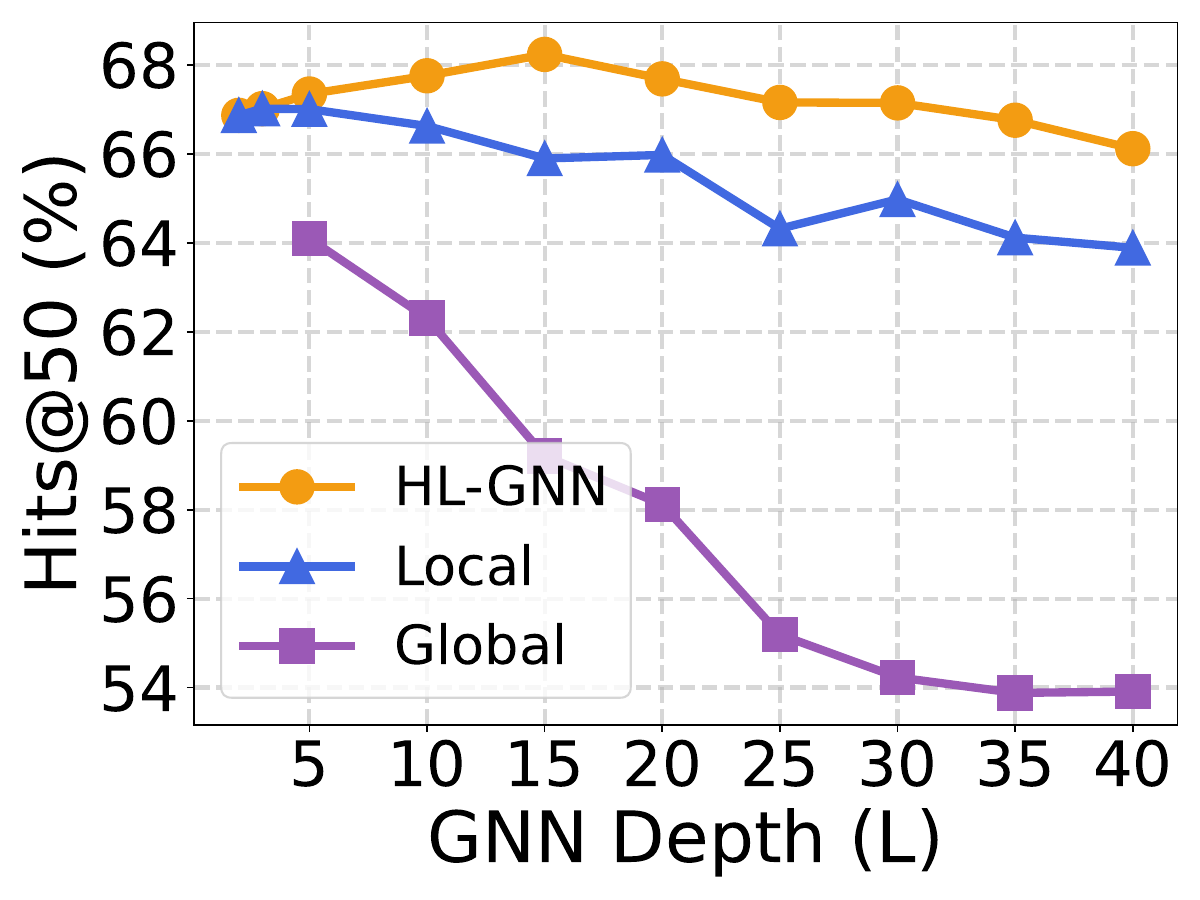}
        \caption{\texttt{ogbl-collab}.}
    \end{subfigure}
    \caption{Ablation study on information ranges. We compare HL-GNN with two GNN variants, focusing on either local or global topological information, with different GNN depths.}
    \label{fig:output}
\end{figure}

\subsubsection{Sufficient model depths}
In HL-GNN, achieving sufficient model depth is crucial for learning global heuristics and capturing long-range dependencies. Our model can effectively reach a depth of around 20 layers without performance deterioration, as shown in Figure~\ref{fig:deep}. In contrast, conventional GNNs often experience a sharp performance drop after just 2 or 3 layers.
For the Planetoid datasets \texttt{Cora} and \texttt{Pubmed}, shallow models yield poor performance, likely due to the absence of global topological information. Conversely, for the OGB datasets \texttt{ogbl-collab} and \texttt{ogbl-ddi}, deeper models (exceeding 15 layers) result in decreased performance, possibly due to the introduction of non-essential global information, which dilutes the crucial local information needed for accurate predictions.

\begin{figure}[tbp]
    \centering
    \begin{subfigure}{0.49\linewidth}
        \includegraphics[width=\textwidth]{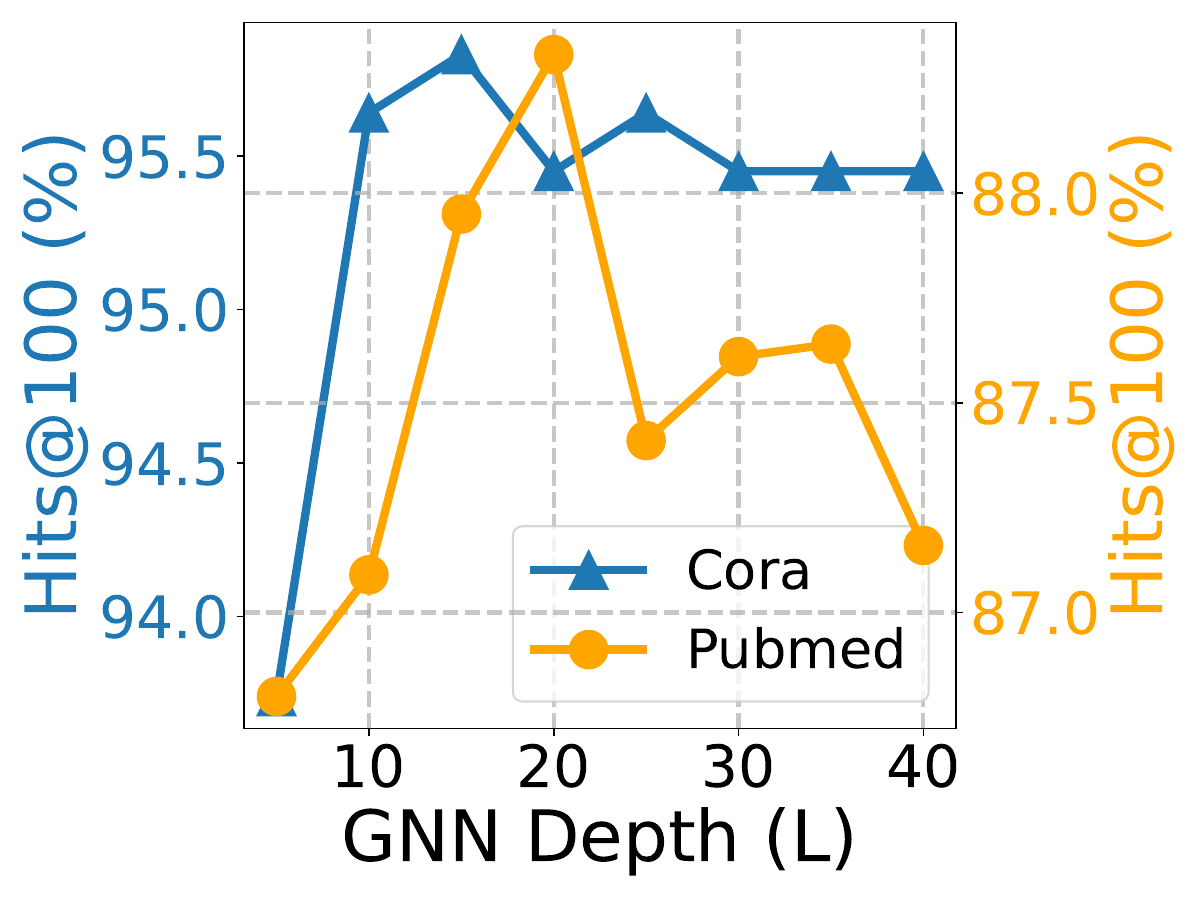}
        \caption{Planetoid datasets.}
    \end{subfigure}
    \begin{subfigure}{0.49\linewidth}
        \includegraphics[width=\textwidth]{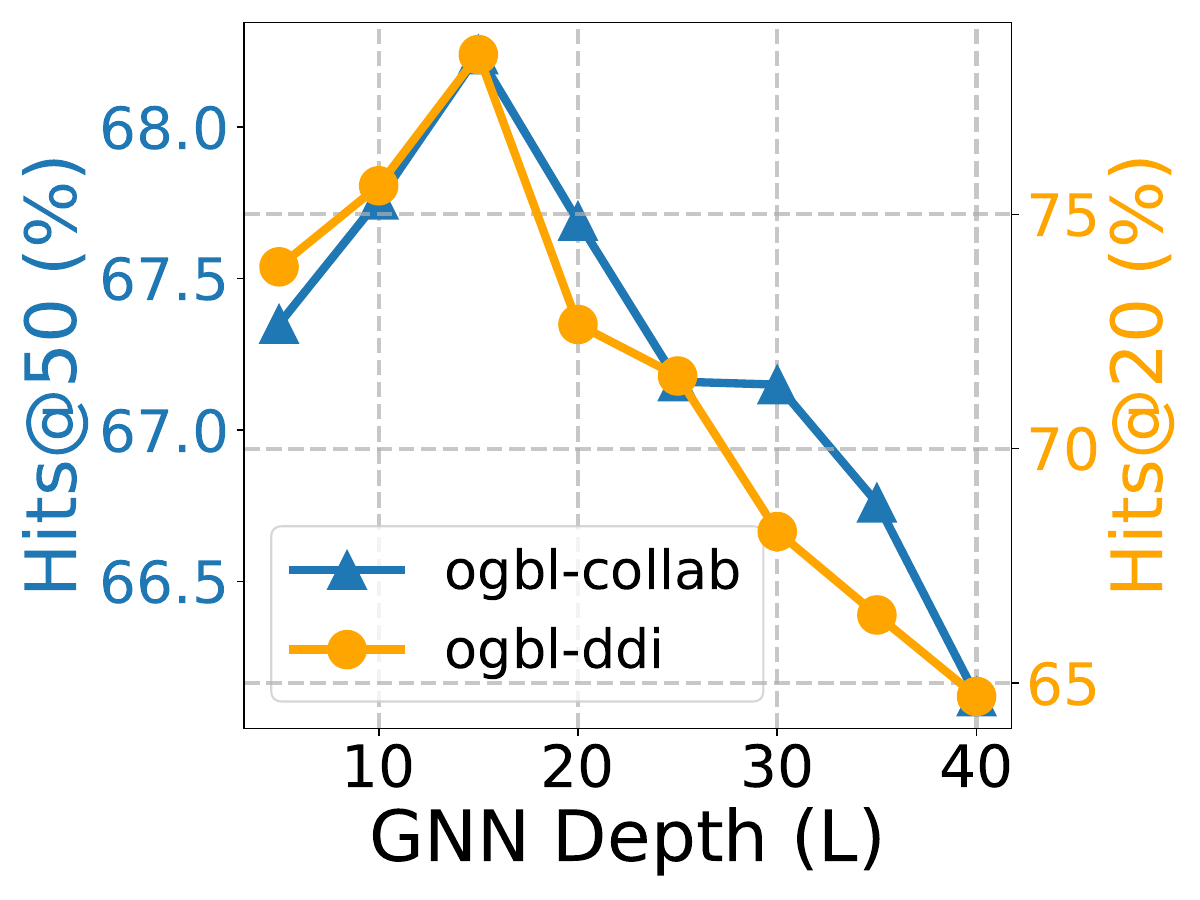}
        \caption{OGB datasets.}
    \end{subfigure}
    \caption{Test performance on the Planetoid and OGB datasets with different GNN depths.}
    \label{fig:deep}
\end{figure}

\subsection{Efficiency Analysis}
\subsubsection{Time efficiency}
Our HL-GNN demonstrates exceptional time efficiency with the lowest time complexity, as indicated in Table~\ref{tab:comparison}. 
The wall time for a single training epoch is provided in Table~\ref{tab:wall_time}. Although HL-GNN generally has a larger depth $L$ compared to conventional GNNs, its experimental wall time per training epoch is comparable to models like GCN and GAT. In practice, HL-GNN requires slightly more time than GCN or GAT due to its increased depth. However, HL-GNN is several orders of magnitude faster than heuristic-inspired GNN methods such as SEAL, NBFNet, and Neo-GNN, thanks to its avoidance of running multiple GNNs and time-consuming manipulations like applying the labeling trick.

\begin{table}[tbp]
    \caption{Wall time per epoch (in seconds) for training HL-GNN compared to other GNN methods. The shortest and second shortest times are marked with \textbf{bold} and \underline{underline}, respectively.}
    \label{tab:wall_time}
    \centering
    \small
    \setlength{\tabcolsep}{7pt}
    \begin{tabular}{lccccc}
    \toprule
       & \texttt{Cora} & \texttt{Citeseer} & \texttt{Pubmed}& \texttt{collab}  &\texttt{ddi}\\
    \midrule
         \textbf{GCN}&  \textbf{0.02}&  \textbf{0.03}&  \textbf{0.4}&  \textbf{5.3}& \textbf{9.2}\\
         \textbf{GAT}&  \underline{0.05}&  0.06&  \underline{0.5}&  \underline{5.8}& \underline{10.4}\\
         \textbf{SEAL}&  28.7&  27.3&  310&   5,130&  15,000\\
         \textbf{NBFNet}&  129&  115&  1,050&  /&  52,000\\
         \textbf{Neo-GNN}&  2.6&  1.4&  19.5&  101& 172\\
         \textbf{BUDDY}& 0.1& 0.1& 0.8& 10.5&17.6\\
    \midrule
         \textbf{HL-GNN}& 0.06&  \underline{0.05}& \underline{0.5}& 6.7& 16.2\\
    \bottomrule
    \end{tabular}
\end{table}

\subsubsection{Parameter efficiency}
HL-GNN only demands a few parameters per layer, with the primary parameter cost incurred by the preprocessing step and the MLP predictor. Table~\ref{tab:parameters} compares the number of parameters in HL-GNN with other GNN methods, clearly highlighting HL-GNN's superior parameter efficiency. Our model stands out as the most parameter-efficient among the listed conventional GNNs and heuristic-inspired GNN methods.

While conventional GNNs excel in efficiency but may lack in performance, and heuristic-inspired GNN methods are effective but time and parameter-intensive, HL-GNN strikes a balance. It consistently achieves top-tier prediction performance on numerous link prediction benchmarks, maintains excellent scalability and time efficiency, and stands out as the most parameter-efficient method. 

\begin{table}[tbp]
    \caption{Number of parameters for HL-GNN compared to other GNN methods. The least and second least number of parameters are marked with \textbf{bold} and \underline{underline}, respectively.}
    \label{tab:parameters}
    \centering
    \small
    \setlength{\tabcolsep}{7pt}
    \begin{tabular}{lccccc}
    \toprule
         & \texttt{Cora} & \texttt{Citeseer} & \texttt{Pubmed}& \texttt{collab}  &\texttt{ddi}\\
    \midrule
         \textbf{GCN}&  \underline{565k}&  \underline{1.15M}&  \underline{326k}&  \underline{231k}& \underline{1.36M}\\
         \textbf{GAT}&  566k&  \underline{1.15M}&  327k&  394k& 1.55M\\
         \textbf{SEAL}&  2.30M&  3.46M&  1.82M&  1.63M& 6.19M\\
         \textbf{NBFNet}&  3.71M&  5.02M&  3.03M&  OOM& 11.04M\\
         \textbf{Neo-GNN}&  631k&  1.21M&  392k&  297k& \underline{1.36M}\\
         \textbf{BUDDY}& 2.52M& 4.85M& 1.57M& 1.19M&2.71M\\
    \midrule
         \textbf{HL-GNN}& \textbf{433k}&  \textbf{1.01M}&  \textbf{194k}& \textbf{99k}& \textbf{1.22M}\\
    \bottomrule
    \end{tabular}
\end{table}

\subsection{Interpretability Analysis}
\subsubsection{Generalized heuristics and learned weights}
\label{para:interpret}
Leveraging the capabilities of the unified formulation, we can derive generalized heuristics by analyzing the learned parameters of HL-GNN. 
The generalized heuristics and learned weights $\beta^{(l)}$ provide insights into the graph-structured data. 
The learned weights $\beta^{(l)}$ are visually depicted in Figure~\ref{fig:weights}.
Due to space constraints, the formulas for the generalized heuristics are provided in Appendix~\ref{para:learned_heuristics}.
 
For the \texttt{Cora} and \texttt{Citeseer} datasets, the learned weights monotonically decrease, indicating that the graph filter serves as a low-pass filter. The weight $\beta^{(0)}$ has the largest magnitude, suggesting that crucial information is primarily contained in node features. Local topological information from nearby neighbors plays a major role, while global topological information from distant nodes serves as a complementary factor.
Conversely, for the \texttt{ogbl-collab} and \texttt{ogbl-ddi} datasets, the weights do not monotonically increase or decrease. Instead, they experience a significant change, especially in the first 5 layers, indicating that the graph filter serves as a high-pass filter. 
The weight $\beta^{(2)}$ has the largest magnitude, suggesting that crucial information lies in local topology rather than node features. Moreover, for large values of $l$ on the \texttt{ogbl-collab} and \texttt{ogbl-ddi} datasets, the weights $\beta^{(l)}$ become negative, suggesting that global topological information from distant nodes compensates for excessive information from nearby neighbors.
The learnable weights $\beta^{(l)}$ govern the trade-off between node features and topological information, enabling the adaptive integration of multi-range topological information.

\begin{figure}[tbp]
    \centering
    \vspace{-5px}
    \begin{subfigure}{0.49\linewidth}
        \centering
        \includegraphics[width=\textwidth]{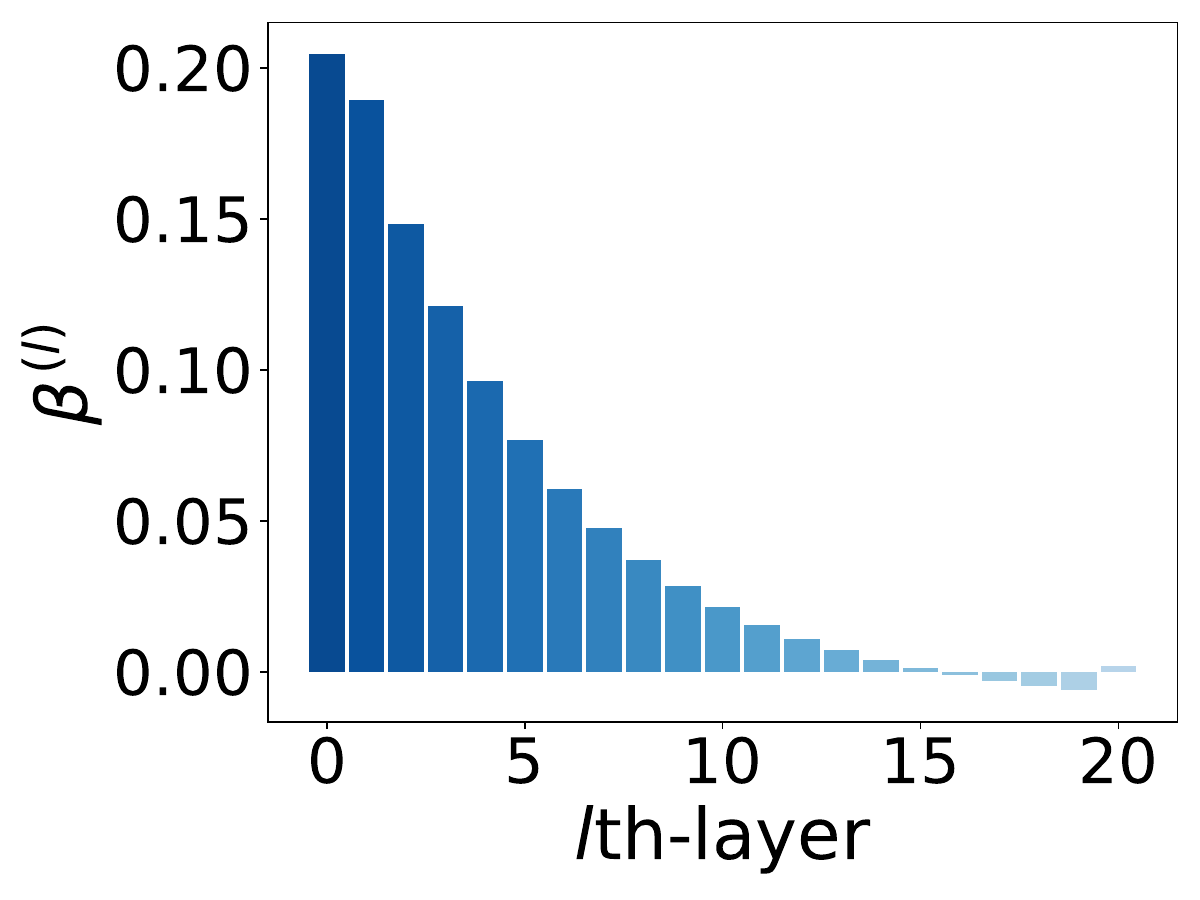}
        \caption{\texttt{Cora}.}
    \end{subfigure}
    \begin{subfigure}{0.49\linewidth}
        \centering
        \includegraphics[width=\textwidth]{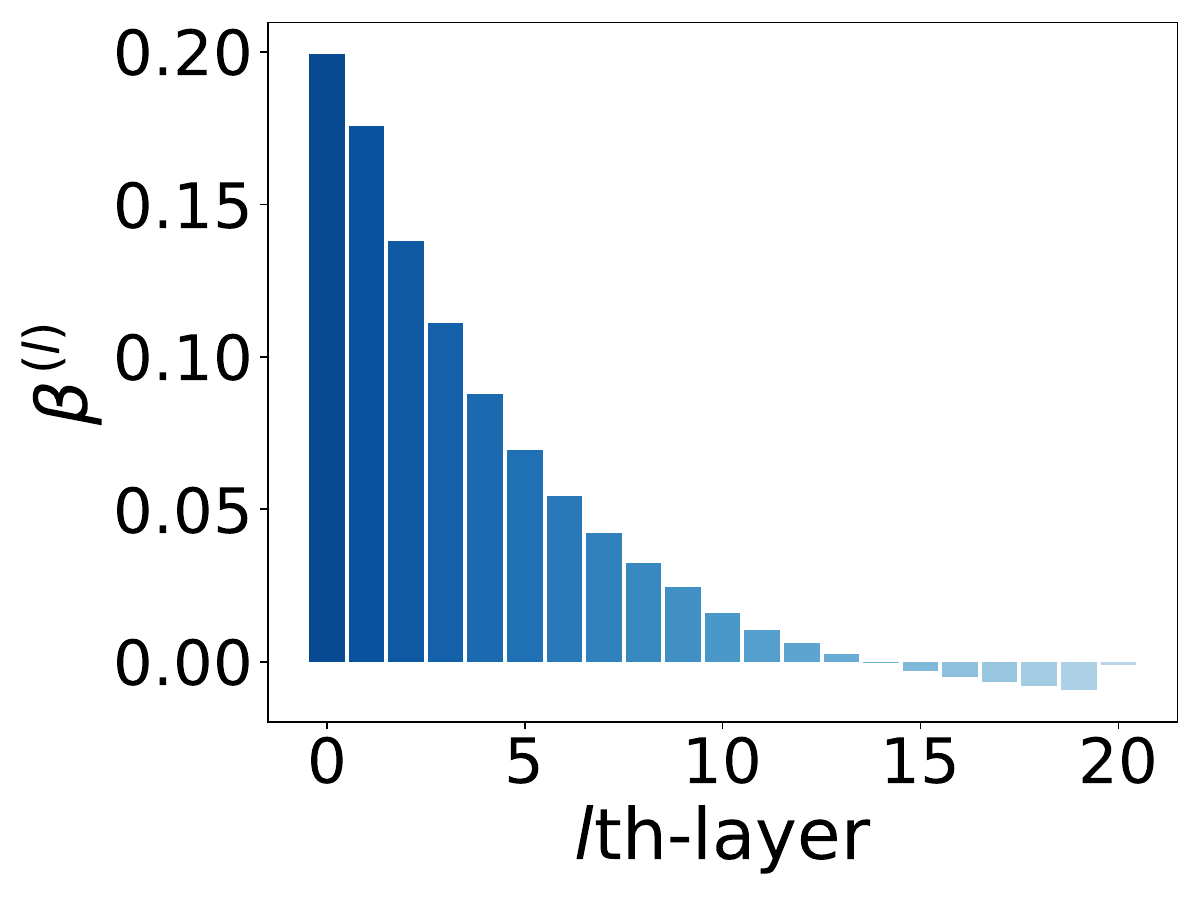}
        \caption{\texttt{Citeseer}.}
    \end{subfigure}
    \begin{subfigure}{0.49\linewidth}
        \centering
        \includegraphics[width=\textwidth]{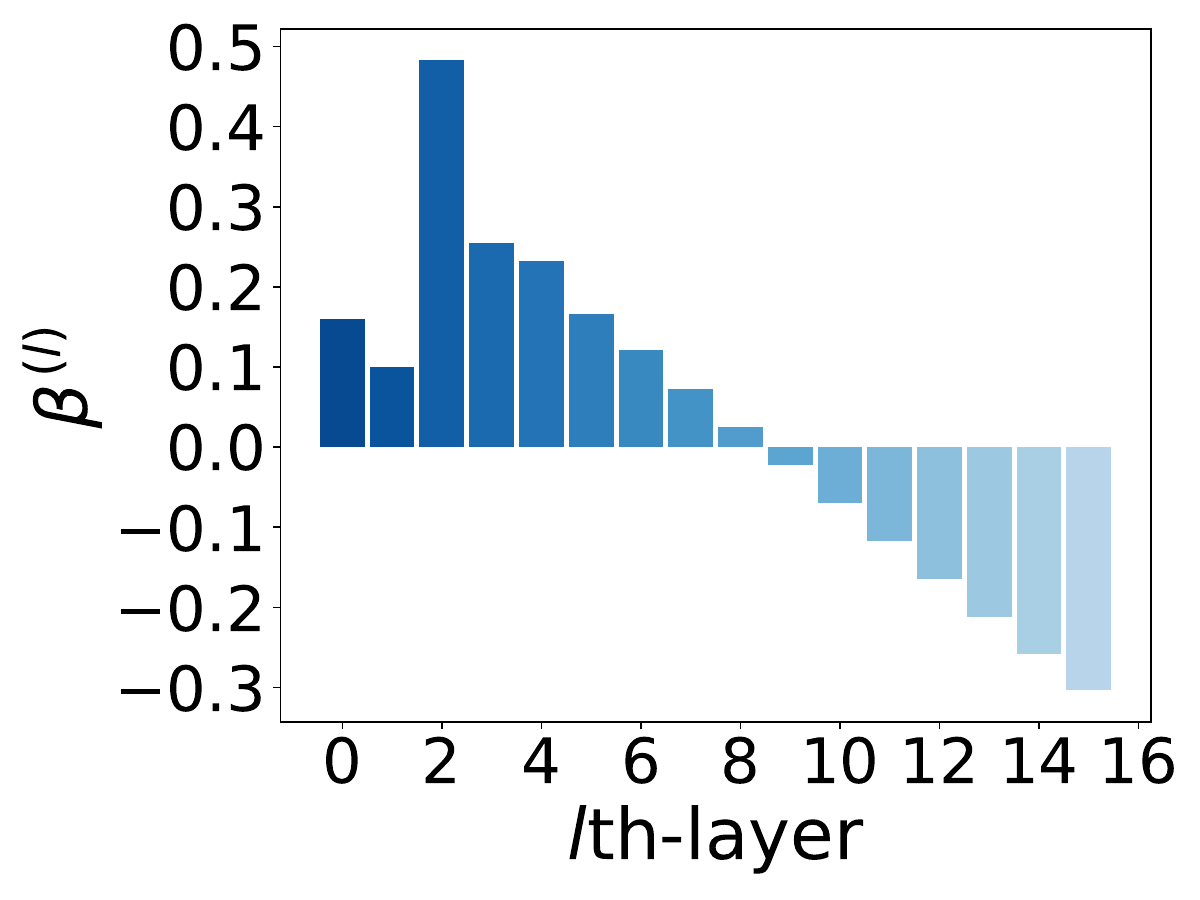}
        \caption{\texttt{ogbl-collab}.}
    \end{subfigure}
    \begin{subfigure}{0.49\linewidth}
        \centering
        \includegraphics[width=\textwidth]{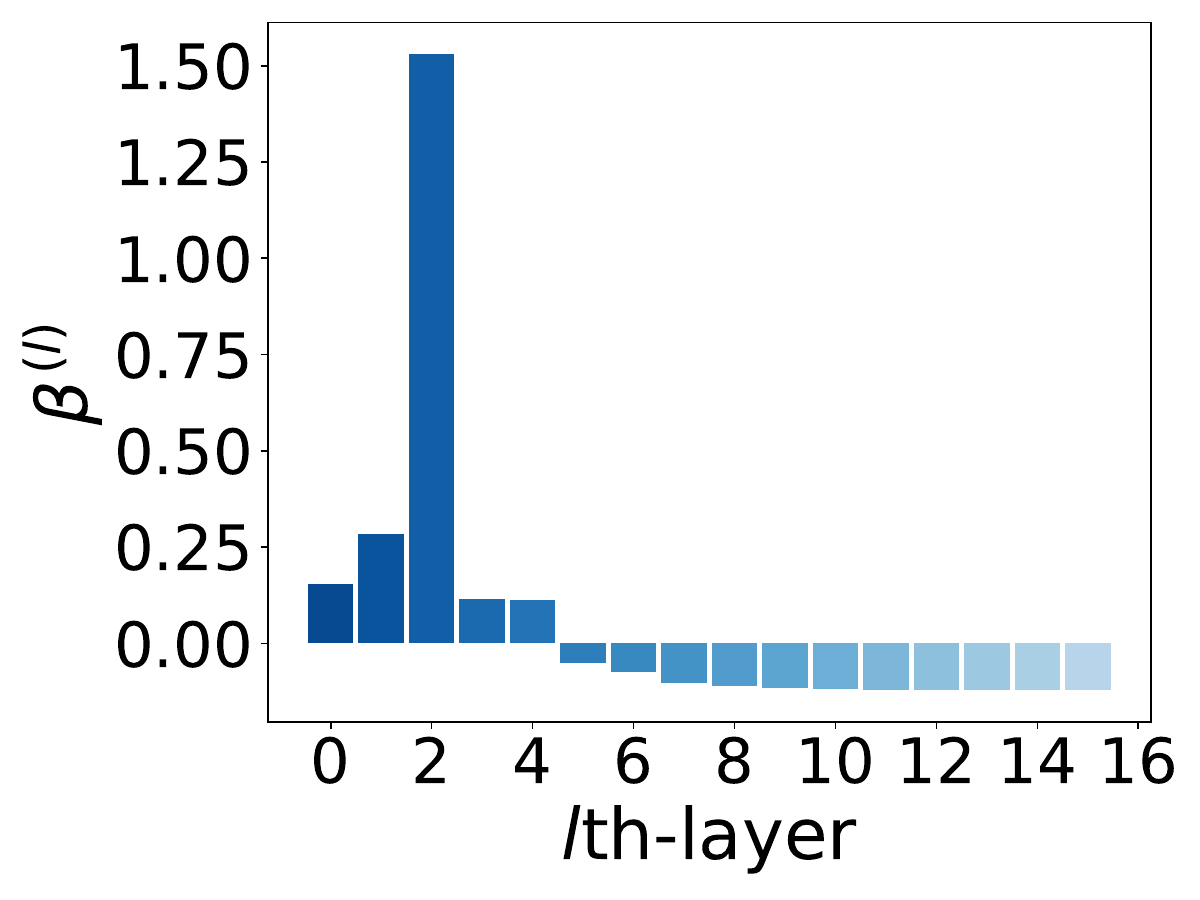}
        \caption{\texttt{ogbl-ddi}.}
    \end{subfigure}
    \caption{Learned weights $\beta^{(l)}$ with $L=20$ for the \texttt{Cora} and \texttt{Citeseer} datasets, and $L=15$ for the \texttt{ogbl-collab} and \texttt{ogbl-ddi} datasets.}
    \label{fig:weights}
\end{figure}

\subsubsection{Leveraging generalized heuristics}

With the generalized heuristic for each dataset, there is no need to train a GNN and an MLP predictor from scratch. Instead, we can simply follow the generalized heuristic and train an MLP predictor only, which is significantly more efficient than training from scratch. The performance of training the MLP alone is comparable to training from scratch, but it converges more quickly. 
The training time for training from scratch versus training only the MLP is shown in Table~\ref{tab:time_mlp}, while the performance details are provided in Appendix~\ref{para:mlp_only}. 
The slight decrease in performance can likely be attributed to the fact that, when training the GNN and MLP together, the gradients flow through both blocks, allowing them to adapt to each other. In contrast, training the MLP alone limits its ability to capture complex interactions between the two blocks.

\begin{table}[tbp]
    \caption{Total training time (in seconds) for training from scratch compared to training only the MLP predictor using the generalized heuristics.}
    \label{tab:time_mlp}
    \centering
    \small
    \setlength{\tabcolsep}{6pt}
    \begin{tabular}{lccccc}
    \toprule
          & \texttt{Cora} & \texttt{Citeseer} & \texttt{Pubmed}& \texttt{collab}  &\texttt{ddi}\\
    \midrule
         \textbf{From Scratch}&  6.3&  5.6&  150&  5,360& 8,100\\
         \textbf{Predictor Only}&  2.8&  2.1&  0.8&  823& 572\\
    \bottomrule
    \end{tabular}
\end{table}

\subsection{Case Study}
We construct two synthetic datasets, a triangular network, and a hexagonal network, to assess HL-GNN's ability to learn the most effective heuristic and obtain the desired range of information.
The triangular network consists of 1000 nodes, with every three nodes forming a triangle. As each pair of nodes shares two common neighbors, we anticipate that the learned heuristic would resemble a local heuristic focusing on 2-hop information. The learned weights are presented in Figure~\ref{fig:case_study}, with $\beta^{(2)}$ having the largest magnitude, corresponding to a local heuristic.

The hexagonal network also comprises 1000 nodes, with every six nodes forming a hexagon. Here, we expect the learned heuristic to resemble a global heuristic focusing on 5-hop information. As shown in Figure~\ref{fig:case_study}, the weight $\beta^{(5)}$ has the largest magnitude, corresponding to a global heuristic. In both cases, HL-GNN demonstrates its ability to adaptively learn the most effective heuristic based on the specific topology.
This also emphasizes the importance of developing a formulation that can effectively accommodate both local and global heuristics.

\begin{figure}[tbp]
    \centering
    \begin{subfigure}{0.49\linewidth}
        \centering
        \includegraphics[width=\textwidth]{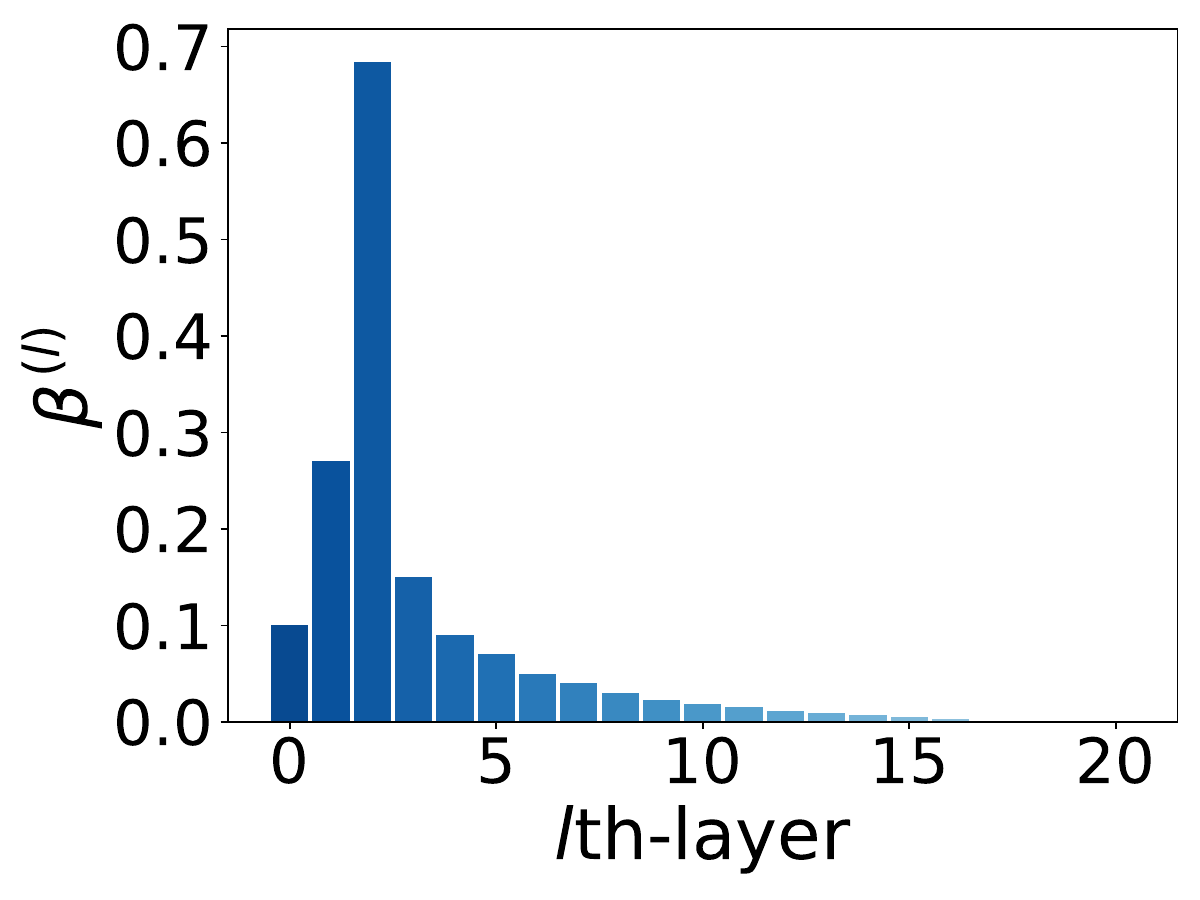}
        \caption{Triangular network.}
    \end{subfigure}
    \begin{subfigure}{0.49\linewidth}
        \centering
        \includegraphics[width=\textwidth]{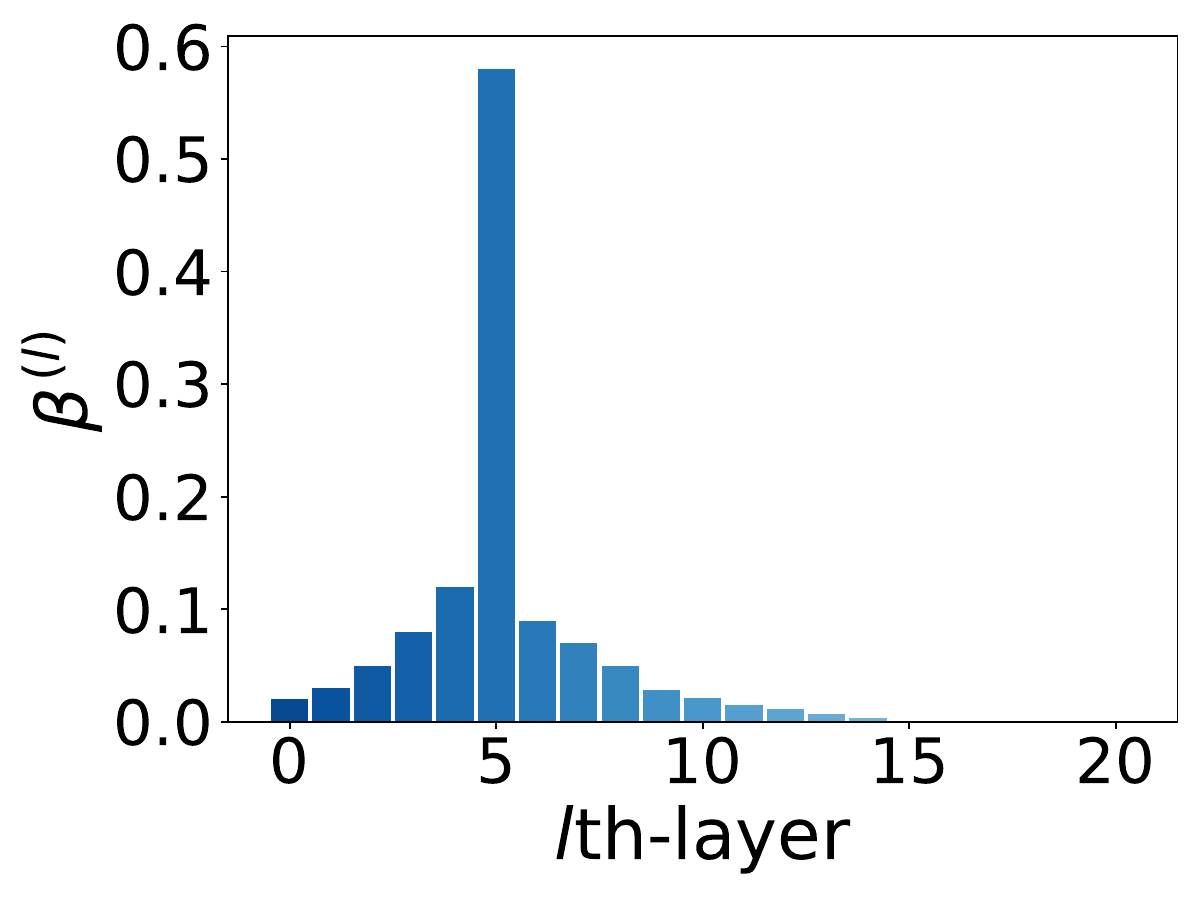}
        \caption{Hexagonal network.}
    \end{subfigure}
    \caption{Learned weights $\beta^{(l)}$ with $L=20$ for the synthetic triangular and hexagonal networks.}
    \label{fig:case_study}
\end{figure}

\section{Conclusion}
We introduce a unified formulation that accommodates and generalizes both local and global heuristics using propagation operators and weight parameters. Additionally, we propose HL-GNN, which efficiently implements this formulation. HL-GNN combines intra-layer propagation and inter-layer connections, allowing the integration of multi-range topological information. Experiments demonstrate that HL-GNN achieves state-of-the-art performance and efficiency.
This study is confined to undirected graphs; for directed graphs, we preprocess them by converting them into undirected graphs. Extensions to multi-relational graphs, such as knowledge graphs, are left for future work.

\section*{Acknowledgment}

This work is supported by National Key Research and Development Program of China (under Grant No.2023YFB2903904),
National Natural Science Foundation of China (under Grant No.92270106)
and Beijing Natural Science Foundation (under Grant No.4242039).


\clearpage
\bibliographystyle{ACM-Reference-Format}
\balance
\bibliography{references}

\newpage
\appendix
\clearpage

\section{Proofs}
\subsection{Proofs of Matrix Forms of Heuristics and Proposition~\ref{prop:formulation-heuristic}}
\label{para:heuristic_proofs}
We provide detailed derivations of the matrix forms of heuristics and demonstrate their alignment with the unified heuristic formulation under specific configurations.

\begin{lemma}
    Common Neighbors (CN), denoted as 
    \[s_\mathrm{CN}(i, j) = |\Gamma_i \cap \Gamma_j| =(\A^2)_{i, j},\]
    conforms to the formulation when \(\mA^{(1)}=\mA^{(2)} = \A\), \(\beta^{(0)}= \beta^{(1)} = 0\), \(\beta^{(2)} = 1\), \(L = 2\).
\end{lemma}
\begin{proof}
    Using the adjacency matrices, we can derive the matrix form of the heuristic CN as follows:
    \begin{equation}
    s_\mathrm{C N}(i, j) = |\Gamma_i \cap \Gamma_j| = \sum_{k \in \mathcal{V}} \tilde{a}_{i k} \tilde{a}_{k j} = (\A^2)_{i, j}.
    \end{equation}
    Using the settings \(\mA^{(1)}=\mA^{(2)} = \A\), \(\beta^{(0)}= \beta^{(1)} = 0\), \(\beta^{(2)} = 1\), \(L = 2\) in Equation~\eqref{eq:formulation} ensures that \(\boldsymbol{H}_{i, j}=h(i,j)=s_\mathrm{C N}(i, j)\).
\end{proof}

\begin{lemma}
    The Local Leicht-Holme-Newman Index (LLHN), denoted as 
    \[s_{\mathrm{LLHN}}(i, j) = \frac{|\Gamma_i \cap \Gamma_j|}{\tilde{d}_i \tilde{d}_j}=(\Ars \Acs)_{i, j},\] 
    conforms to the formulation when \(\mA^{(1)} = \Ars\), \(\mA^{(2)} = \Acs\), \(\beta^{(0)}=\beta^{(1)} = 0\), \(\beta^{(2)} = 1\), \(L = 2\).
\end{lemma}
\begin{proof}
    Using the adjacency matrices, we can derive the matrix form of the heuristic LLHN as follows:
    \begin{equation}
    s_{\mathrm{LLHN}}(i, j) = \frac{|\Gamma_i \cap \Gamma_j|}{\tilde{d}_i \tilde{d}_j} = \sum_{k \in \mathcal{V}} \frac{\tilde{a}_{i k}}{\tilde{d}_i} \frac{\tilde{a}_{k j}}{\tilde{d}_j} = (\Ars \Acs)_{i, j}.
    \end{equation}
     Using the settings \(\mA^{(1)} = \Ars\), \(\mA^{(2)} = \Acs\), \(\beta^{(0)}=\beta^{(1)} = 0\), \(\beta^{(2)} = 1\), \(L = 2\) in Equation~\eqref{eq:formulation} ensures that \(\boldsymbol{H}_{i, j}=h(i,j)=s_\mathrm{LLHN}(i, j)\).
\end{proof}

\begin{lemma}
    The Resource Allocation Index (RA), denoted as 
    \[s_{\mathrm{RA}}(i,j)\\=\sum_{k \in \Gamma_i \cap \Gamma_j} \frac{1}{\tilde{d}_k}=(\Acs \A)_{i, j}=(\A \Ars)_{i, j},\]
    conforms to the formulation under two configurations: (i) \(\mA^{(1)} = \Acs\), \(\mA^{(2)} = \A\), \(\beta^{(0)}=\beta^{(1)} = 0\), \(\beta^{(2)} = 1\), \(L = 2\); or (ii) \(\mA^{(1)} = \A\), \(\mA^{(2)} = \Ars\), \(\beta^{(0)}=\beta^{(1)} = 0\), \(\beta^{(2)} = 1\), \(L = 2\).
\end{lemma}
\begin{proof}
    Using the adjacency matrices, we can derive the matrix form of the heuristic RA as follows:
    \begin{equation}
    s_{\mathrm{RA}}(i, j)=\sum_{k \in \Gamma_i \cap \Gamma_j} \frac{1}{\tilde{d}_k}=\sum_{k \in \mathcal{V}} \frac{\tilde{a}_{i k}}{\tilde{d}_k} \tilde{a}_{k j}=(\Acs \A)_{i, j}.
    \end{equation}    
    Alternatively, the matrix form of the heuristic RA can also be expressed as:
    \begin{equation}
    s_{\mathrm{RA}}(i, j)=\sum_{k \in \Gamma_i \cap \Gamma_j} \frac{1}{\tilde{d}_k}=\sum_{k \in \mathcal{V}} {\tilde{a}_{i k}} \frac{\tilde{a}_{k j}}{\tilde{d}_k}=(\A \Ars)_{i, j}.
    \end{equation}
    Using the two settings: (i) \(\mA^{(1)} = \Acs\), \(\mA^{(2)} = \A\), \(\beta^{(0)}=\beta^{(1)} = 0\), \(\beta^{(2)} = 1\), \(L = 2\); or (ii) \(\mA^{(1)} = \A\), \(\mA^{(2)} = \Ars\), \(\beta^{(0)}=\beta^{(1)} = 0\), \(\beta^{(2)} = 1\), \(L = 2\) in Equation~\eqref{eq:formulation} ensures that \(\boldsymbol{H}_{i, j}=h(i,j)=s_\mathrm{RA}(i, j)\).
\end{proof}

\begin{lemma}
\label{lemma:KI}
    The Katz Index (KI), denoted as 
    $$s_{\mathrm{KI}}(i, j)=\sum_{l=1}^{\infty} \gamma^l |{\textnormal{paths}}_{i,j}^l|=\left(\sum_{l=1}^{\infty} \gamma^l \A^l\right)_{i,j},$$
    where ${\textnormal{paths}}_{i,j}^l$ is the set of length-$l$ paths between nodes $i$ and $j$, $\gamma$ is a damping factor, conforms to the formulation when $\mA^{(m)} = \A$ for $m \geq 1$, $\beta^{(0)} = 0$, $ \beta^{(l)}=\gamma^l$ for $l \geq 1$, $L = \infty$.
\end{lemma}
\begin{proof}
    Using the adjacency matrices, we can derive the matrix form of the heuristic KI as follows:
    \begin{equation}
    s_{\mathrm{KI}}(i, j)=\sum_{l=1}^{\infty} \gamma^l |{\textnormal{paths}}_{i,j}^l|=\sum_{l=1}^{\infty} \gamma^l(\A^l)_{i,j}=\left(\sum_{l=1}^{\infty} \gamma^l \A^l\right)_{i,j}.
    \end{equation}
    Using the settings $\mA^{(m)} = \A$ for $m \geq 1$, $\beta^{(0)} = 0$, $ \beta^{(l)}=\gamma^l$ for $l \geq 1$, $L = \infty$ in Equation~\eqref{eq:formulation} ensures that \(\boldsymbol{H}_{i, j}=h(i,j)=s_\mathrm{KI}(i, j)\).
\end{proof}

\begin{lemma}
    The Global Leicht-Holme-Newman Index (GLHN), denoted as 
    $$s_{\mathrm{GLHN}}(i, j)=\sum_{l=0}^{\infty} \phi^l |{\textnormal{paths}}_{i,j}^l|= \left(\boldsymbol{I}_N+\sum_{l=1}^{\infty} \phi^l \A^l \right)_{i,j},$$
    where ${\textnormal{paths}}_{i,j}^l$ is the set of length-$l$ paths between nodes $i$ and $j$, $\phi$ is a damping factor, conforms to the formulation when $\mA^{(m)} = \A$ for $m \geq 1$, $\beta^{(0)} = 1$, $ \beta^{(l)}=\phi^l$ for $l \geq 1$, $L = \infty$.
\end{lemma}
\begin{proof}
    Using the adjacency matrices, we can derive the matrix form of the heuristic GLHN as follows:
    \begin{equation}
    \begin{aligned}
    s_{\mathrm{GLHN}}(i, j)&=\sum_{l=0}^{\infty} \phi^l |{\textnormal{paths}}_{i,j}^l|=(\boldsymbol{I}_N)_{i,j}+ \sum_{l=1}^{\infty} \phi^l(\A^l)_{i,j}\\
    &=\left(\boldsymbol{I}_N+\sum_{l=1}^{\infty} \phi^l \A^l \right)_{i,j}.
    \end{aligned}
    \end{equation}
    Using the settings $\mA^{(m)} = \A$ for $m \geq 1$, $\beta^{(0)} = 1$, $ \beta^{(l)}=\phi^l$ for $l \geq 1$, $L = \infty$ in Equation~\eqref{eq:formulation} ensures that \(\boldsymbol{H}_{i, j}=h(i,j)=s_\mathrm{GLHN}(i, j)\).
\end{proof}

\begin{lemma}
\label{lemma:RWR}
    The Random Walk with Restart (RWR), denoted as $$s_{\mathrm{RWR}}(i,j)=[\boldsymbol{\pi}_i(\infty)]_j =\left(\sum_{l=0}^{\infty}(1-\alpha)\alpha^l \Ars^l\right)_{i,j},$$ 
    conforms to the formulation when $\mA^{(m)} = \Ars$ for $m \geq 1$,  $\beta^{(l)}=(1-\alpha)\alpha^l$ for $l \geq 0$, $L = \infty$.
\end{lemma}
\begin{proof}
    Random Walk with Restart (RWR) calculates the stationary distribution of a random walker starting at $i$, who iteratively moves to a random neighbor of its current position with probability $\alpha$ or returns to $i$ with probability $1-\alpha$. 
    Let $\boldsymbol{\pi}_i$ denote the probability vector of reaching any node starting a random walk from node $i$.
    Let $\Acs$ be the transition matrix and $(\tilde{\boldsymbol{A}}_{\mathrm{cs}})_{ij} = (\tilde{\boldsymbol{A}} \tilde{\boldsymbol{D}}^{-1})_{ij}=\frac{1}{\tilde{d}_j}$ if $(i, j) \in \mathcal{E}$ and $(\tilde{\boldsymbol{A}}_{\mathrm{cs}})_{ij}=0$ otherwise.
    Let $\boldsymbol{e}_i$ be an indicator vector with the $i^{\text{th}}$ element being 1 and others being 0.
    The probability of reaching each node can be iteratively approximated by
    \begin{equation}
    \label{iter}
    \boldsymbol{\pi}_i(t)=\alpha \Acs \boldsymbol{\pi}_i(t-1)+(1-\alpha) \boldsymbol{e}_i.
    \end{equation}
    The stationary distribution of this probability vector can be calculated as $t \rightarrow \infty$,
    \begin{equation}
    \label{infinity}
    \boldsymbol{\pi}_i(\infty)=(1-\alpha) (\boldsymbol{I}_N-\alpha\Acs)^{-1} \boldsymbol{e}_i.
    \end{equation}
    
    We calculate the closed-form solution to the iterative function of Equation~\eqref{iter} with $\boldsymbol{\pi}_i(0)=0$:
    \begin{equation}
    \boldsymbol{\pi}_i(t)=\sum_{l=0}^{t-1} (1-\alpha)\alpha^l \Acs^l \boldsymbol{e}_i,
    \end{equation}
    where $t \geq 1$.
    We can set $t \rightarrow \infty$,
    \begin{equation}
    \boldsymbol{\pi}_i(\infty)=\sum_{l=0}^{\infty}(1-\alpha) \alpha^l \Acs^l \boldsymbol{e}_i=(1-\alpha) (\boldsymbol{I}_N-\alpha\Acs)^{-1} \boldsymbol{e}_i,
    \end{equation}
    which conforms to the stationary distribution of Equation~\eqref{infinity}.
    By substituting the indicator vector $\boldsymbol{e}_i$ with the identity matrix $\boldsymbol{I}_N$, we obtain the RWR matrix 
    \begin{equation}
    \boldsymbol{\Pi}(\infty)=\sum_{l=0}^{\infty}(1-\alpha)\alpha^l \Acs^l, 
    \end{equation}
    whose element $(j,i)$ specifies the probability of the random walker starts from node $i$ and locates at node $j$ in the stationary state.
    The heuristic for link $(i,j)$ is given by this random walk probability from $i$ to $j$, denoted as $[\boldsymbol{\pi}_i(\infty)]_j$ (or $[\boldsymbol{\pi}_i(\infty)]_j+[\boldsymbol{\pi}_j(\infty)]_i$ for symmetry):
    \begin{equation}
    \begin{aligned}
        s_{\mathrm{RWR}}(i,j)&=[\boldsymbol{\pi}_i(\infty)]_j\\&=\left[\boldsymbol{\Pi}(\infty)\right]_{j,i}
        \\&=\left(\sum_{l=0}^{\infty}(1-\alpha)\alpha^l \Acs^l\right)_{j,i}\\&=\left(\sum_{l=0}^{\infty}(1-\alpha)\alpha^l \Ars^l\right)_{i,j}.
    \end{aligned}
    \end{equation}
     Using the settings $\mA^{(m)} = \Ars$ for $m \geq 1$,  $\beta^{(l)}=(1-\alpha)\alpha^l$ for $l \geq 0$, $L = \infty$ in Equation~\eqref{eq:formulation} ensures that \(\boldsymbol{H}_{i, j}=h(i,j)=s_\mathrm{RWR}(i, j)\).
\end{proof}

\begin{lemma}
    The Local Path Index (LPI), denoted as 
    $$s_{\mathrm{LPI}}(i, j)=\sum_{l=2}^L \gamma^{l-2} |{\textnormal{paths}}_{i,j}^l|=\left(\sum_{l=2}^L \gamma^{l-2} \A^l\right)_{i,j},$$
    where ${\textnormal{paths}}_{i,j}^l$ is the set of length-$l$ paths between nodes $i$ and $j$, $\gamma$ is a damping factor, conforms to the formulation when $\mA^{(m)} = \A$ for $1 \leq m \leq L$, $\beta^{(0)} = \beta^{(1)} =0$, $ \beta^{(l)}=\gamma^{l-2}$ for $2 \leq l \leq L$.
\end{lemma}
\begin{proof}
    Following Lemma~\ref{lemma:KI}, using the adjacency matrices, we can derive the matrix form of the heuristic LPI as follows:
    \begin{equation}
    s_{\mathrm{LPI}}(i, j)=\sum_{l=2}^L \gamma^{l-2} |{\textnormal{paths}}_{i,j}^l|=\sum_{l=2}^L \gamma^{l-2}(\A^l)_{i,j}=\left(\sum_{l=2}^L \gamma^{l-2} \A^l\right)_{i,j}.
    \end{equation}
    Using the settings $\mA^{(m)} = \A$ for $1 \leq m \leq L$, $\beta^{(0)} = \beta^{(1)} =0$, $ \beta^{(l)}=\gamma^{l-2}$ for $2 \leq l \leq L$ in Equation~\eqref{eq:formulation} ensures that \(\boldsymbol{H}_{i, j}=h(i,j)=s_\mathrm{LPI}(i, j)\).
\end{proof}

\begin{lemma}
    The Local Random Walks (LRW), denoted as $$s_{\mathrm{LRW}}(i,j)=\frac{\tilde{d}_i}{2M}[\boldsymbol{\pi}_i(L)]_j =\left(\sum_{l=0}^{L-1}\frac{\tilde{d}_i}{2M}(1-\alpha)\alpha^l \Ars^l\right)_{i,j},$$ 
    conforms to the formulation when $\mA^{(m)} = \Ars$ for $1 \leq m \leq L-1$,  $\beta^{(l)}=\frac{\tilde{d}_i}{2M} (1-\alpha)\alpha^l$ for $0 \leq l \leq L-1$.
\end{lemma}
\begin{proof}
    Following Lemma~\ref{lemma:RWR},
    \begin{equation}
    \boldsymbol{\pi}_i(L)=\sum_{l=0}^{L-1} (1-\alpha)\alpha^l \Acs^l \boldsymbol{e}_i,
    \end{equation}
    where $L \geq 1$.
    By substituting the indicator vector $\boldsymbol{e}_i$ with the identity matrix $\boldsymbol{I}_N$, we obtain the RWR matrix 
    \begin{equation}
    \boldsymbol{\Pi}(L)=\sum_{l=0}^{L-1}(1-\alpha)\alpha^l \Acs^l.
    \end{equation}
    Using the adjacency matrices, we can derive the matrix form of the heuristic LRW as follows:
    \begin{equation}
    \begin{aligned}
        s_{\mathrm{LRW}}(i,j)&=\frac{\tilde{d}_i}{2M}[\boldsymbol{\pi}_i(L)]_j\\
        &=\frac{\tilde{d}_i}{2M}\left[\boldsymbol{\Pi}(L)\right]_{j,i}\\
        &=\left(\sum_{l=0}^{L-1}\frac{\tilde{d}_i}{2M}(1-\alpha)\alpha^l \Acs^l\right)_{j,i}\\
        &=\left(\sum_{l=0}^{L-1}\frac{\tilde{d}_i}{2M}(1-\alpha)\alpha^l \Ars^l\right)_{i,j}.
    \end{aligned}
    \end{equation}
     Using the settings $\mA^{(m)} = \Ars$ for $1 \leq m \leq L-1$,  $\beta^{(l)}=\frac{\tilde{d}_i}{2M} (1-\alpha)\alpha^l$ for $0 \leq l \leq L-1$ in Equation~\eqref{eq:formulation} ensures that \(\boldsymbol{H}_{i, j}=h(i,j)=s_\mathrm{LRW}(i, j)\).
\end{proof}

\begin{proposition*}
    Our formulation can accommodate a broad spectrum of local and global heuristics with propagation operators $\mathbb{A}^{(m)}$ for $1 \leq m \leq L$, weight parameters $\beta^{(l)}$ for $0 \leq l \leq L$, and maximum order $L$. 
\end{proposition*}
\begin{proof}
    Building upon above lemmas, it becomes evident that a diverse range of local and global heuristics can be represented in matrix form. Our formulation is capable of accommodating a wide range of local and global heuristics under specific configurations. This serves to prove Proposition~\ref{prop:formulation-heuristic}.
\end{proof}

\subsection{Proof of Proposition~\ref{prop:relationship}} 
\label{para:relationship_proof}
\begin{proposition*}
    The relationship between the learned representations $\boldsymbol{Z}$ in Equation~\eqref{eq:HLGNN} and the heuristic formulation $\boldsymbol{H}$ in Equation~\eqref{eq:formulation} is given by $\boldsymbol{Z} = \boldsymbol{H}\boldsymbol{X}$, where $\boldsymbol{X}$ is the node feature matrix.
\end{proposition*}

\begin{proof}
    Using Equation~\eqref{eq:HLGNN}, we have:
    \begin{equation}
    \boldsymbol{Z}^{(l)} = \mathbb{A}^{(l)} \boldsymbol{Z}^{(l-1)},
    \end{equation}
    with the initial condition 
    \begin{equation}
    \boldsymbol{Z}^{(0)} = \boldsymbol{X}.
    \end{equation}
    Iteratively using the equation, we get:
    \begin{equation}
    \label{eq:proof_iter}
    \boldsymbol{Z}^{(l)} = \prod_{m=1}^l \mathbb{A}^{(m)} \boldsymbol{X}= \prod_{m=0}^l \mathbb{A}^{(m)} \boldsymbol{X}.
    \end{equation}
    The second equality is due to $\mathbb{A}^{(0)}=\boldsymbol{I}_N$.
    Using Equation~\eqref{eq:HLGNN}, we can express \(\boldsymbol{Z}\) as:
    \begin{equation}
    \boldsymbol{Z} = \sum_{l=0}^L \beta^{(l)} \boldsymbol{Z}^{(l)}.
    \end{equation}
    Substituting Equation~\eqref{eq:proof_iter}, we get:
    \begin{equation}
    \boldsymbol{Z} = \sum_{l=0}^L \left(\beta^{(l)} \prod_{m=0}^l \mathbb{A}^{(m)} \boldsymbol{X}\right)=\left(\sum_{l=0}^L \left(\beta^{(l)} \prod_{m=0}^l \mathbb{A}^{(m)} \right)\right)\boldsymbol{X}.
    \end{equation}
    Now, examining Equation~\eqref{eq:formulation}, it becomes evident that:
    \begin{equation}
    \boldsymbol{Z} = \boldsymbol{H}\boldsymbol{X}.
    \end{equation}
\end{proof}

\section{Additional Experimental Details}
\subsection{Detailed Information about Datasets}
\label{para:dataset_statistics}

\begin{table*}[tbp]
    \caption{Statistics of the Planetoid, Amazon, and OGB datasets used in the experiments.}
    \label{tab:dataset_statistics}
    \centering
    \begin{tabular}{lccccccccc}
    \toprule &  \texttt{Cora} & \texttt{Citeseer} & \texttt{Pubmed}& \texttt{Photo}&\texttt{Computers} & \texttt{collab}  &\texttt{ddi} & \texttt{ppa}&\texttt{citation2}\\
    \midrule \textbf{\# Nodes} & 2,708 & 3,327 & 18,717 &   7,650&13,752&235,868 & 4,267  & 576,289&2,927,963\\
    \textbf{\# Edges} & 5,278 & 4,676 & 44,327 &   238,162&491,722&1,285,465& 1,334,889 & 30,326,273&30,561,187\\
    \textbf{\# Features}& 1,433& 3,703& 500& 745& 767& 128&/ & 58&128\\
    \textbf{Dataset Split}& random & random & random &   random&random&time& time& throughput&protein\\
    \textbf{Average Degree}& 3.90& 2.81& 4.73&   62.26&71.51&5.45 & 312.84  & 52.62&10.44\\
    \bottomrule
    \end{tabular}
\end{table*}

We employ a diverse set of benchmark datasets to comprehensively evaluate the link prediction performance of HL-GNN. Each dataset captures distinct aspects of relationships and presents unique challenges for link prediction.
We utilize nine datasets from three sources: Planetoid~\cite{planetoid}, Amazon~\cite{amazon}, and OGB~\cite{OGB}. 
The Planetoid datasets include \texttt{Cora}, \texttt{Citeseer}, and \texttt{Pubmed}. The Amazon datasets include \texttt{Photo} and \texttt{Computers}. 
The OGB datasets include \texttt{ogbl-collab}, \texttt{ogbl-ddi}, \texttt{ogbl-ppa}, and \texttt{ogbl-citation2}. 
The statistics of each dataset are shown in Table~\ref{tab:dataset_statistics}.

We introduce each dataset briefly: 
\begin{itemize}
    \item \texttt{Cora}: This dataset represents a citation network. Nodes in the graph correspond to documents, and edges denote citations between documents. Each document is associated with a bag-of-words representation.
    \item \texttt{Citeseer}: Similar to \texttt{Cora}, \texttt{Citeseer} is also a citation network. Nodes represent scientific articles, and edges signify citations. The node features include bag-of-words representations and information about the publication.
    \item \texttt{Pubmed}: Another citation network, \texttt{Pubmed} comprises nodes corresponding to scientific publications, and edges represent citations. Node features consist of binary indicators of word occurrences in the documents.
    \item \texttt{Photo}: This dataset captures a social network, where nodes represent users and edges denote friendships. Node features encompass visual and textual information from user profiles.
    \item \texttt{Computers}: Representing a co-purchase network, \texttt{Computers} has nodes representing computers and edges indicating co-purchases. Node features involve TF-IDF representations of computer terms.
    \item \texttt{ogbl-collab}: This dataset is an author collaboration network. Nodes are authors, and edges denote collaborations. The node features contain bag-of-words representations of authors' publications.
    \item \texttt{ogbl-ddi}: \texttt{ogbl-ddi} stands for drug-drug interaction, and this network models interactions between drugs. Edges indicate interactions, and the dataset contains information about drug structures.
    \item \texttt{ogbl-ppa}: This dataset represents a protein-protein association network. Nodes correspond to proteins, and edges indicate associations based on various biological evidence. Node features include protein sequences and structural information.
    \item \texttt{ogbl-citation2}: This dataset is a citation network that focuses on the dynamics of citations. Nodes represent papers, and edges indicate citation relationships. Node features consist of paper metadata.
\end{itemize}

These datasets are widely used for evaluating link prediction methods, each presenting unique challenges based on the nature of the relationships captured in the respective graphs.

\subsection{Detailed Information about Experiments}
\label{para:exp_detail}

\begin{table*}[tbp]
    \caption{The detailed hyperparameter configurations used in the experiments.}
    \label{tab:hyperparameter}
    \centering
    \begin{tabular}{llccccccccc}
    \toprule  &Hyperparameter & \texttt{Cora} & \texttt{CiteSeer} & \texttt{PubMed} & \texttt{Photo} & \texttt{Computers} & \texttt{collab} & \texttt{ddi}  & \texttt{ppa}&\texttt{citation2}\\
    \midrule
     \multirow{5}{*}{\textbf{GNN}} &GNN \# layers& 20 & 20 & 20 & 20 & 20 & 15 & 15  & 15  &15  \\
      & GNN hidden dim.& 1,433 & 3,703 & 500 & 745 & 767 & 256 & 512  & 512  &256\\
      & Node emb. dim. & - & - & - & - & - & 256 & 512  & 256&64\\
      & $\beta$ initialization & RWR & RWR & KI & RWR & RWR & KI & KI  & KI  &KI  \\
      & Init. parameter & $\alpha=0.2$ & $\alpha=0.2$ & $\gamma=0.2$ & $\alpha=0.2$ & $\alpha=0.2$ & $\gamma=0.5$ & $\gamma=0.5$  & $\gamma=0.5$  &$\gamma=0.6$\\
      \midrule
      \multirow{2}{*}{\textbf{MLP}} &MLP \# layers& 3 & 2 & 3 & 3 & 3 & 2 & 2  & 2  &2  \\
      & MLP hidden dim.& 8,192 & 8,192 & 512 & 512 & 512 & 256 & 512  & 512  &256\\
      \midrule
      \multirow{5}{*}{\textbf{Learning}}&Optimizer & Adam & Adam & Adam & Adam & Adam & Adam & Adam & Adam &Adam \\
      & Loss function& BCE& BCE& BCE& BCE& BCE& AUC& AUC& AUC&AUC\\
      & Learning rate& 0.001 & 0.001 & 0.001 & 0.001 & 0.001 & 0.001 & 0.001  & 0.001  &0.001  \\
      &\# Epochs & 100 & 100 & 300 & 200 & 200 & 800 & 500  & 500  &100\\
      &Dropout rate & 0.5 & 0.5 & 0.6 & 0.6 & 0.6 & 0.3 & 0.3  & 0.5&0.3  \\
    \bottomrule
    \end{tabular}
\end{table*}

In our experiments, we employ a set of hyperparameters for training HL-GNN across various datasets. The hyperparameter details are outlined in Table~\ref{tab:hyperparameter}. All experiments can be performed using an A100 (80GB) GPU. Hyperparameters are tuned via random search using Weights and Biases, involving 50 runs. The hyperparameters yielding the highest validation accuracy were selected, and results are reported on the test set.

The number of layers in HL-GNN is set to 20 for the Planetoid and Amazon datasets, and set to 15 for the OGB datasets. We keep the hidden dimensionality of HL-GNN the same as the raw node features on the Planetoid and Amazon datasets. Additionally, we leverage node embeddings on the OGB datasets to enhance the node representations. The dimensions of node embeddings are set to 256, 512, 256, and 64 for \texttt{ogbl-collab}, \texttt{ogbl-ddi}, \texttt{ogbl-ppa}, and \texttt{ogbl-citation2}, respectively.
 
Inspired by heuristics KI and RWR presented in Table~\ref{tab:heursitics}, we explore different initialization strategies for $\beta^{(l)}$. We adopt KI initialization as $\beta^{(l)} = \gamma^l$ and RWR initialization as $\beta^{(l)} = (1-\alpha)\alpha^l$. We consistently adopt KI initialization for the OGB datasets and RWR initialization for the Planetoid and Amazon datasets (except for $\texttt{Pubmed}$, which utilizes KI initialization). We do not impose constraints on $\beta^{(l)}$, allowing them to take positive or negative values.

The MLP utilized in HL-GNN serves as the predictor and comprises 3 layers for most Planetoid and Amazon datasets, and comprises 2 layers for the OGB datasets. 
We employ the Adam optimizer with a fixed learning rate of 0.001 for all datasets. Training is carried out for a specified number of epochs, ranging from 100 to 800, depending on the dataset. Dropout regularization is optimized among \{0.3, 0.4, 0.5, 0.6, 0.7\}.

\section{Additional Experiments}
\subsection{More Ablation Studies}
\subsubsection{Different initialization strategies}
We investigate the impact of various initial values for $\beta^{(l)}$ on the final performance in this section. 
Inspired by heuristics KI and RWR in Table~\ref{tab:heursitics}, we explore different initialization strategies, including:
\begin{itemize}
\item KI initialization: $\beta^{(l)} = \gamma^l$.
\item RWR initialization: $\beta^{(l)} = (1-\alpha)\alpha^l$.
\item Random initialization: $\beta^{(l)} \sim \mathcal{N}(0, 1)$.
\item Uniform initialization: $\beta^{(l)} = 1/L$.
\item Reverse-KI initialization: $\beta^{(l)} = \gamma^{L-l}$.
\item Final-layer initialization: $\beta^{(L)}=1$, with others set to 0.
\end{itemize}

The results of these different initialization strategies are presented in Figure~\ref{fig:initial}, with performance comparisons on four datasets. Notably, the KI and RWR initialization strategies outperform the others, underscoring the importance of emphasizing short-range dependencies, given their direct topological relevance.

It is noteworthy that even without extensive tuning, KI initialization at $\gamma=0.5$ for the OGB datasets and RWR initialization at $\alpha=0.2$ for the Planetoid and Amazon datasets consistently demonstrate strong performance. In practical applications, we consistently adopt KI initialization for the OGB datasets and RWR initialization for the Planetoid and Amazon datasets (except for $\texttt{Pubmed}$, which utilizes KI initialization) in our experiments. The initialization strategies for each dataset are summarized in Table~\ref{tab:hyperparameter}.

\begin{figure}[tbp]
    \centering
    \begin{subfigure}{0.49\linewidth}
        \includegraphics[width=\textwidth]{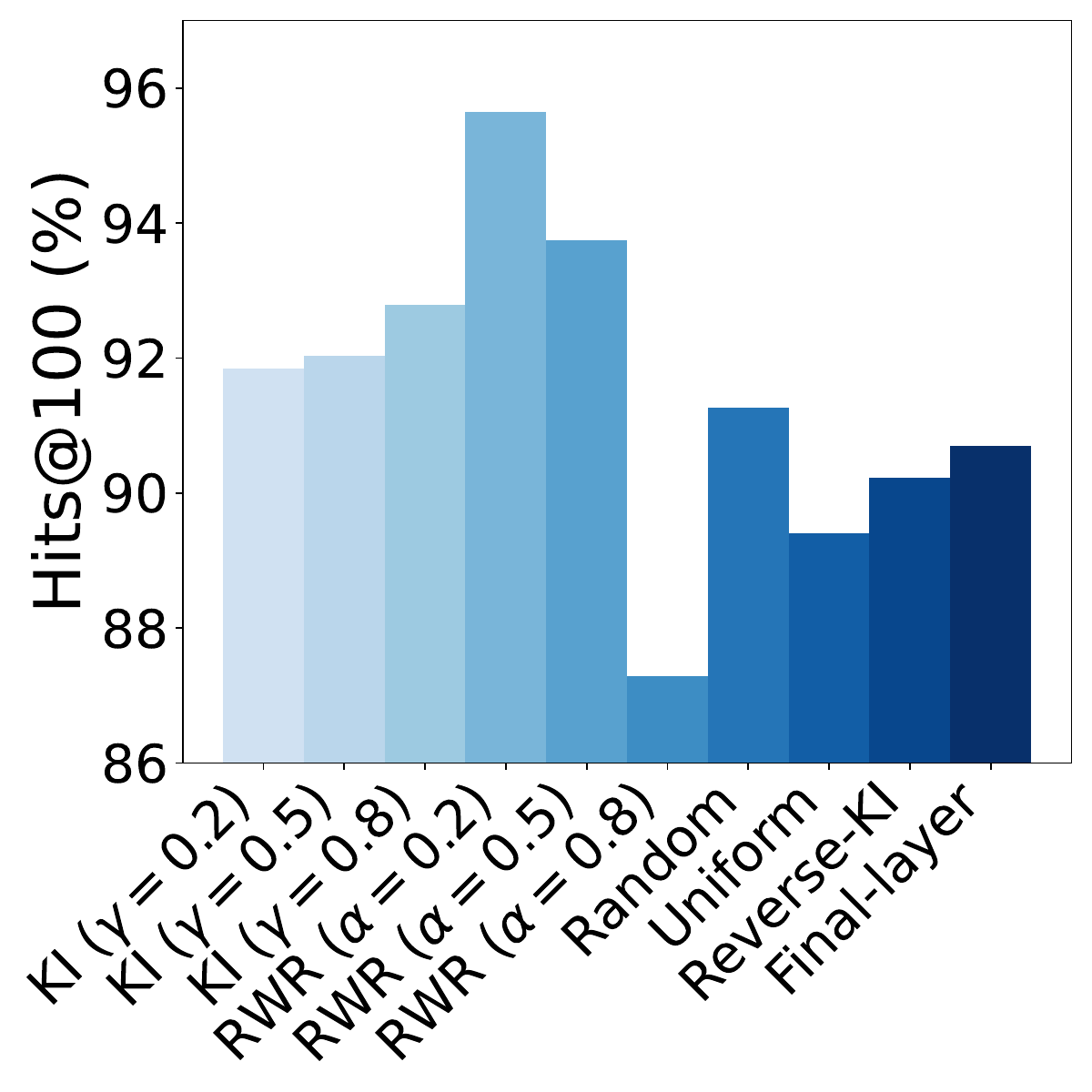}
        \caption{\texttt{Cora}.}
    \end{subfigure}
    \begin{subfigure}{0.49\linewidth}
    \includegraphics[width=\textwidth]{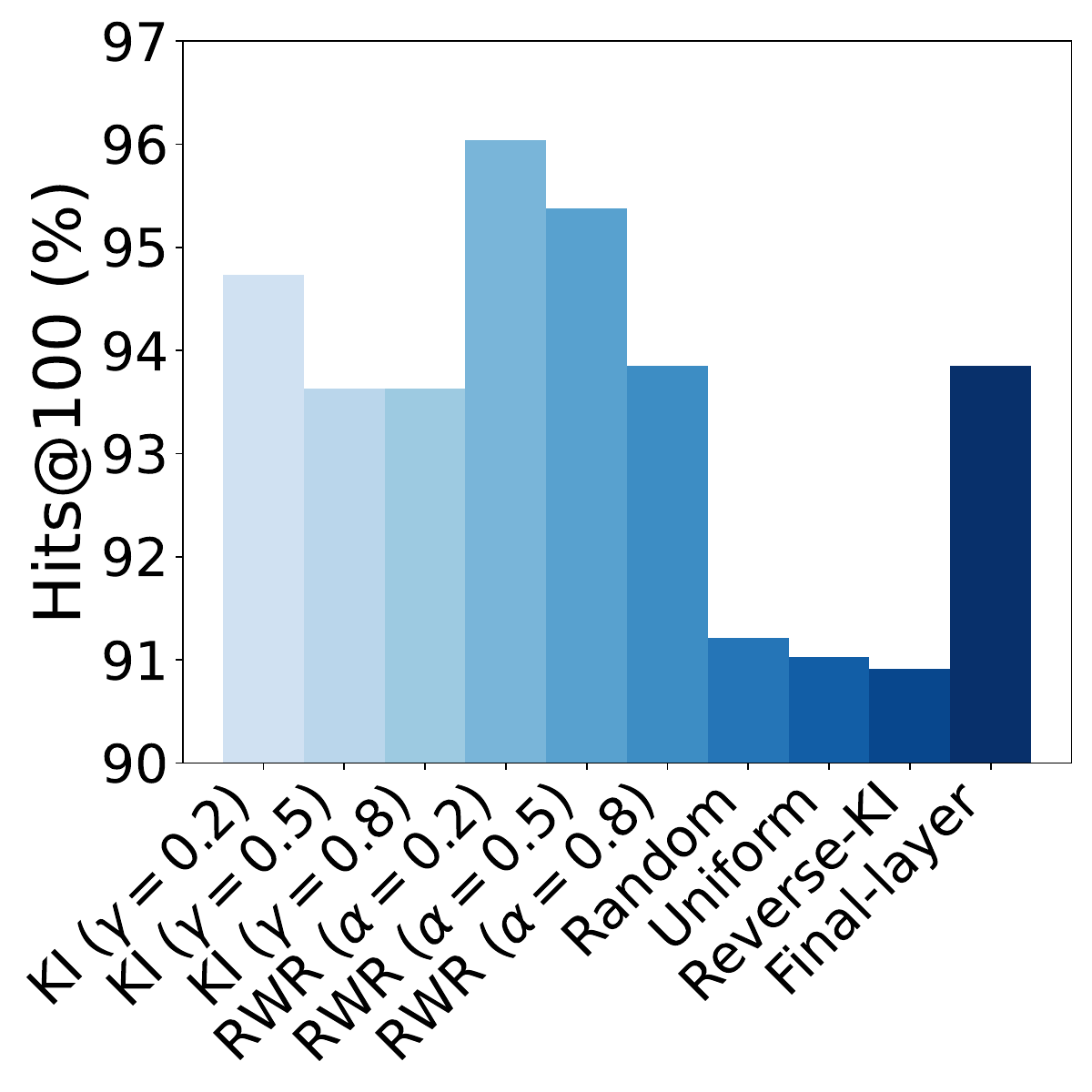}
    \caption{\texttt{Citeseer}.}
    \end{subfigure}
    \begin{subfigure}{0.49\linewidth}
        \includegraphics[width=\textwidth]{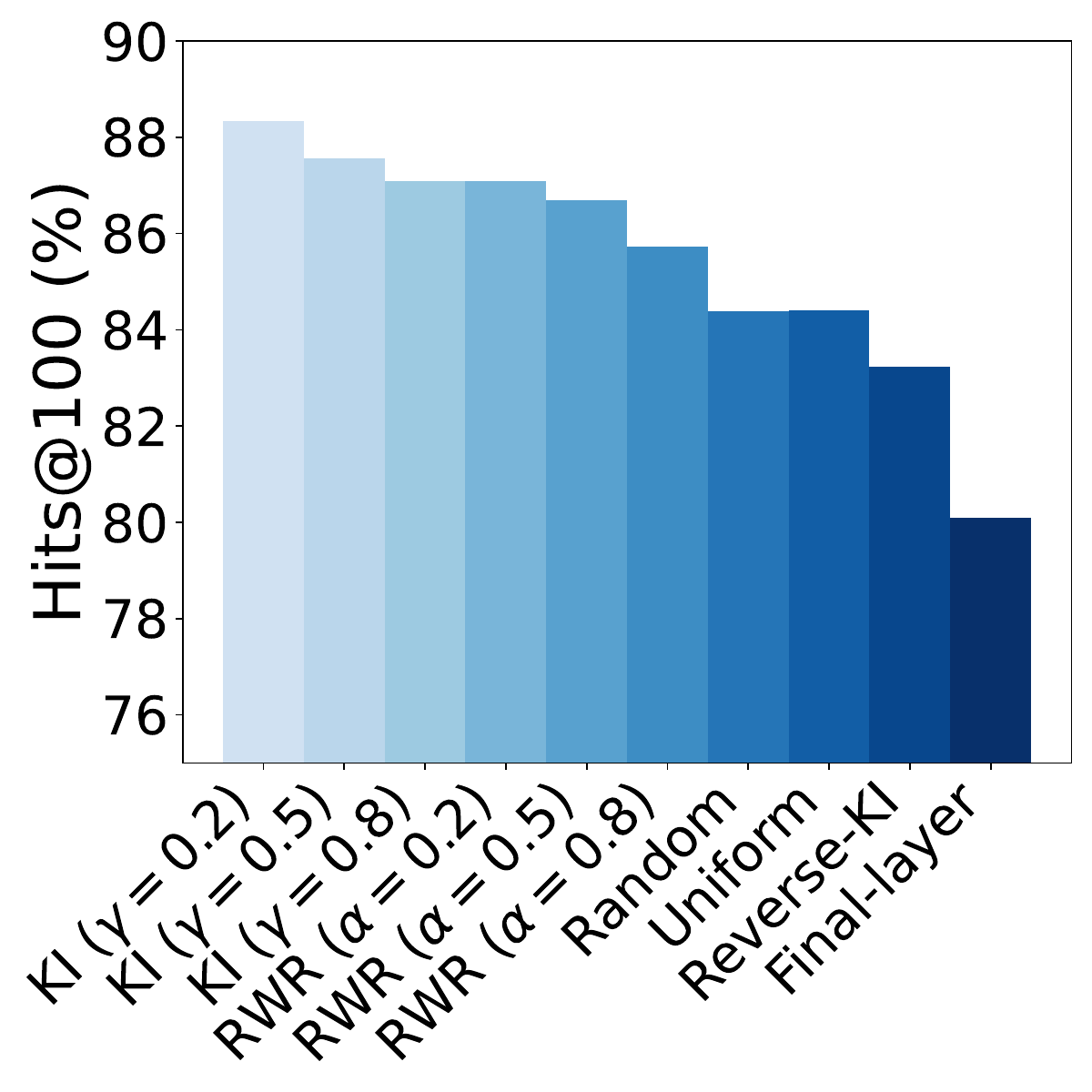}
        \caption{\texttt{Pubmed}.}
    \end{subfigure}
    \begin{subfigure}{0.49\linewidth}
        \includegraphics[width=\textwidth]{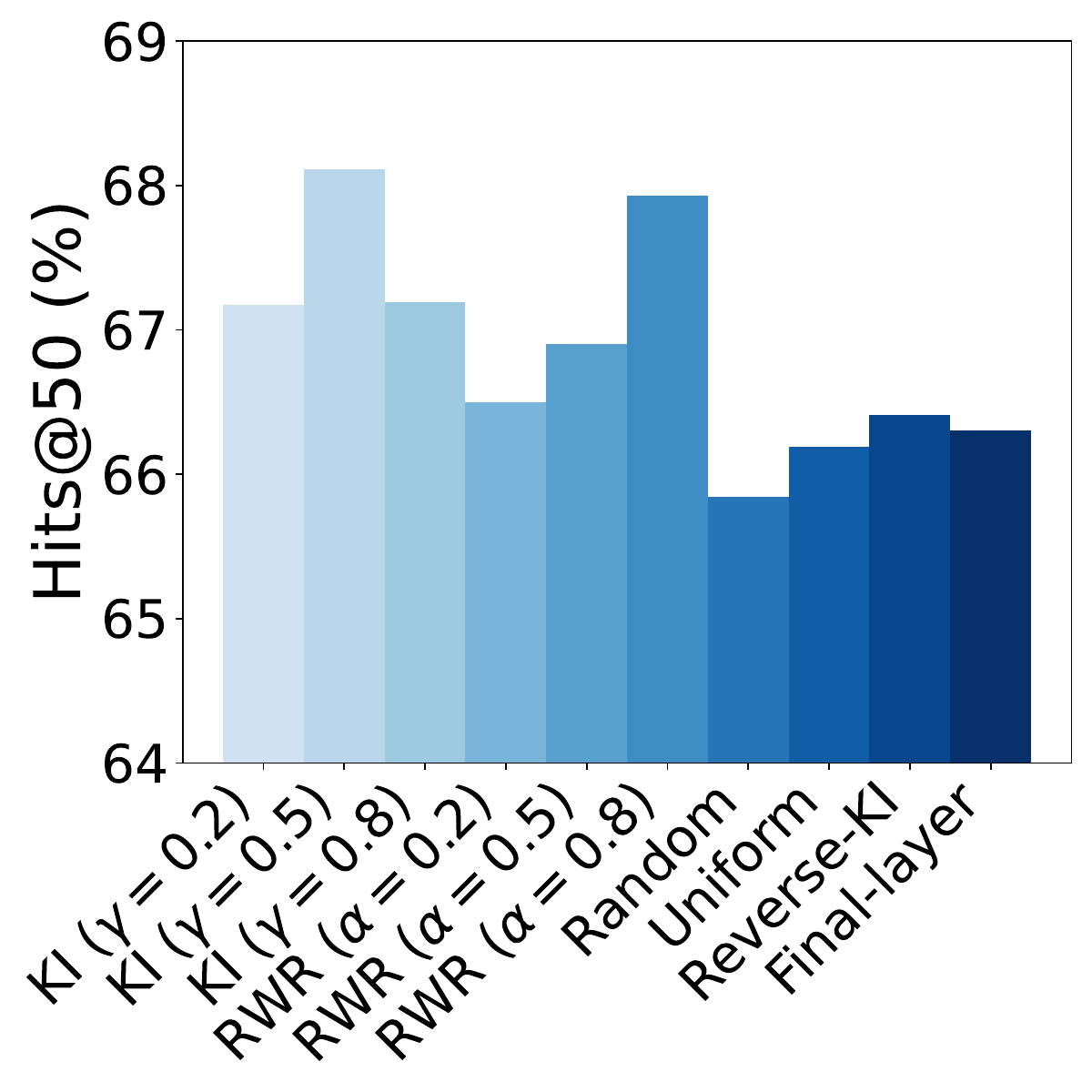}
        \caption{\texttt{ogbl-collab}.}
    \end{subfigure}
    \caption{Ablation study on different initialization strategies across various datasets.}
    \label{fig:initial}
\end{figure}

\subsubsection{Introduction of transformation and activation functions}
To investigate if the introduction of transformation and activation functions could potentially improve the performance of HL-GNN, we conduct additional experiments on these two functions. The results are summarized in Table~\ref{tab:appen_trans}.

\begin{table*}[tbp]
    \caption{Ablation study on the introduction of transformation and activation functions.}
    \label{tab:appen_trans}
    \centering
    \begin{tabular}{lccccccc}
    \toprule
        & \texttt{Cora} & \texttt{Citeseer} & \texttt{Pubmed}& \texttt{Photo}&\texttt{Computers} & \texttt{collab}  &\texttt{ddi}\\
       & Hits@100 & Hits@100 & Hits@100 & AUC &AUC & Hits@50 &Hits@20  \\
    \midrule
     \textbf{HL-GNN (w/ ReLU act.)}& 92.89\footnotesize{$\pm$1.89}& 93.15\footnotesize{$\pm$2.21}& 87.34\footnotesize{$\pm$1.27}& 98.23\footnotesize{$\pm$0.23}&97.79\footnotesize{$\pm$0.17}& 66.29\footnotesize{$\pm$0.71}&78.15\footnotesize{$\pm$2.69}\\
      \textbf{HL-GNN (w/ tran.)}& 87.35\footnotesize{$\pm$2.05}& 88.91\footnotesize{$\pm$2.45}& 83.39\footnotesize{$\pm$1.19}& 95.97\footnotesize{$\pm$0.18}&95.39\footnotesize{$\pm$0.41}& 60.47\footnotesize{$\pm$0.91}&73.04\footnotesize{$\pm$3.17}\\
      \textbf{HL-GNN (w/ ReLU act. \& tran.)}& 86.23\footnotesize{$\pm$1.87}& 86.74\footnotesize{$\pm$2.36}& 82.15\footnotesize{$\pm$1.69}& 95.45\footnotesize{$\pm$0.31}&94.75\footnotesize{$\pm$0.35}& 59.84\footnotesize{$\pm$0.88}&71.53\footnotesize{$\pm$4.46}\\
      \midrule
      \textbf{HL-GNN} & \textbf{94.22\footnotesize{$\pm$1.64}} & \textbf{94.31\footnotesize{$\pm$1.51}} & \textbf{88.15\footnotesize{$\pm$0.38}} & \textbf{99.11\footnotesize{$\pm$0.07}} & \textbf{98.82\footnotesize{$\pm$0.21}}  & \textbf{68.11\footnotesize{$\pm$0.54}} &\textbf{80.27\footnotesize{$\pm$3.98}} \\
    \bottomrule
    \end{tabular}
\end{table*}

The results reveal that the introduction of transformation and activation functions do not lead to an improvement in performance. This outcome can be attributed to the disturbance of learned heuristics by non-linearity, preventing them from retaining the form of matrix multiplications, which is the foundation of our formulation.
Furthermore, the inclusion of transformation matrices significantly increases the number of parameters, posing challenges to the learning process. Consequently, HL-GNN with transformation is unable to maintain the depth of around 20 layers. 
The optimal performance is achieved with a limited depth of less than 5 layers. 
This limitation contributes to the observed inferior performance of HL-GNN (w/ tran.) and HL-GNN (w/ ReLU act. \& tran.).

\subsection{Interpretability of Generalized Heuristics}
\label{para:learned_heuristics}
In this section, we present the learned heuristic formulation $\boldsymbol{H}$ to exhibit the interpretability of generalized heuristics. It's worth noting that each individual heuristic for a link $(i,j)$ can be directly extracted from the $(i,j)$ entry of $\boldsymbol{H}$. For clarity, we simplify Equation~\eqref{eq:relaxation} by omitting $\alpha^{(l)}$ and utilizing $\mathbb{A}^{(l)}=\tilde{\boldsymbol{A}}_{\mathrm{sym}}$. For conciseness, we provide the first ten weights below.

For $\texttt{Cora}$,
    \begin{equation}
    \begin{aligned}
\boldsymbol{H}_{\mathrm{Cora}}&=0.1795\boldsymbol{I}+ 0.1894\tilde{\boldsymbol{A}}_{\mathrm{sym}}+0.1484\tilde{\boldsymbol{A}}_{\mathrm{sym}}^2
+0.1213\tilde{\boldsymbol{A}}_{\mathrm{sym}}^3\\
&+0.0963\tilde{\boldsymbol{A}}_{\mathrm{sym}}^4 +0.0768\tilde{\boldsymbol{A}}_{\mathrm{sym}}^5+0.0606\tilde{\boldsymbol{A}}_{\mathrm{sym}}^6
+0.0477\tilde{\boldsymbol{A}}_{\mathrm{sym}}^7\\
&+0.0371\tilde{\boldsymbol{A}}_{\mathrm{sym}}^8+0.0285\tilde{\boldsymbol{A}}_{\mathrm{sym}}^9+\cdots.
    \end{aligned}
    \end{equation}

For $\texttt{Citeseer}$,    
    \begin{equation}
    \begin{aligned}
    \boldsymbol{H}_{\mathrm{Citeseer}}&=0.1993\boldsymbol{I}+ 0.1759\tilde{\boldsymbol{A}}_{\mathrm{sym}}+0.1380\tilde{\boldsymbol{A}}_{\mathrm{sym}}^2
    +0.1111\tilde{\boldsymbol{A}}_{\mathrm{sym}}^3\\
    &+0.0878\tilde{\boldsymbol{A}}_{\mathrm{sym}}^4+0.0694\tilde{\boldsymbol{A}}_{\mathrm{sym}}^5+ 0.0544\tilde{\boldsymbol{A}}_{\mathrm{sym}}^6+0.0422\tilde{\boldsymbol{A}}_{\mathrm{sym}}^7\\
    &+0.0324\tilde{\boldsymbol{A}}_{\mathrm{sym}}^8+0.0244\tilde{\boldsymbol{A}}_{\mathrm{sym}}^9+\cdots.
    \end{aligned}
    \end{equation}

For $\texttt{Pubmed}$, 
    \begin{equation}
    \begin{aligned}
\boldsymbol{H}_{\mathrm{Pubmed}}&=0.7386\boldsymbol{I}+ 0.5335\tilde{\boldsymbol{A}}_{\mathrm{sym}}+0.3126\tilde{\boldsymbol{A}}_{\mathrm{sym}}^2
+0.2787\tilde{\boldsymbol{A}}_{\mathrm{sym}}^3\\
&+0.2429\tilde{\boldsymbol{A}}_{\mathrm{sym}}^4+0.2252\tilde{\boldsymbol{A}}_{\mathrm{sym}}^5+ 0.2039\tilde{\boldsymbol{A}}_{\mathrm{sym}}^6+0.1884\tilde{\boldsymbol{A}}_{\mathrm{sym}}^7\\
&+0.1722\tilde{\boldsymbol{A}}_{\mathrm{sym}}^8+0.1592\tilde{\boldsymbol{A}}_{\mathrm{sym}}^9+\cdots.
    \end{aligned}
    \end{equation}

For $\texttt{ogbl-collab}$, 
    \begin{equation}
    \begin{aligned}
    \boldsymbol{H}_{\mathrm{collab}}&=0.1599\boldsymbol{I}+ 0.1004\tilde{\boldsymbol{A}}_{\mathrm{sym}}+0.4826\tilde{\boldsymbol{A}}_{\mathrm{sym}}^2
    +0.2545\tilde{\boldsymbol{A}}_{\mathrm{sym}}^3\\
    &+0.2328\tilde{\boldsymbol{A}}_{\mathrm{sym}}^4+0.1663\tilde{\boldsymbol{A}}_{\mathrm{sym}}^5+ 0.1210\tilde{\boldsymbol{A}}_{\mathrm{sym}}^6+0.0721\tilde{\boldsymbol{A}}_{\mathrm{sym}}^7\\
    &+0.0252\tilde{\boldsymbol{A}}_{\mathrm{sym}}^8-0.0224\tilde{\boldsymbol{A}}_{\mathrm{sym}}^9+\cdots.
    \end{aligned}
    \end{equation}

For $\texttt{ogbl-ddi}$, 
    \begin{equation}
    \begin{aligned}
    \boldsymbol{H}_{\mathrm{ddi}}&=0.1532\boldsymbol{I}+ 0.2840\tilde{\boldsymbol{A}}_{\mathrm{sym}}+1.5298\tilde{\boldsymbol{A}}_{\mathrm{sym}}^2
    +0.1154\tilde{\boldsymbol{A}}_{\mathrm{sym}}^3\\
    &+0.1123\tilde{\boldsymbol{A}}_{\mathrm{sym}}^4-0.0507\tilde{\boldsymbol{A}}_{\mathrm{sym}}^5- 0.0739\tilde{\boldsymbol{A}}_{\mathrm{sym}}^6-0.1022\tilde{\boldsymbol{A}}_{\mathrm{sym}}^7\\
    &-0.1102\tilde{\boldsymbol{A}}_{\mathrm{sym}}^8-0.1164\tilde{\boldsymbol{A}}_{\mathrm{sym}}^9+\cdots.
    \end{aligned}
    \end{equation}

\subsection{Leveraging Generalized Heuristics}
\label{para:mlp_only}
\begin{table*}[tbp]
    \caption{Performance comparison between training from scratch and training only the MLP predictor using the generalized heuristics. The format is average score $\pm$ standard deviation. }
    \label{tab:performance_mlp}
    \centering
    \begin{tabular}{lccccccc}
    \toprule
          & \texttt{Cora} & \texttt{Citeseer} & \texttt{Pubmed}& \texttt{Photo}&\texttt{Computers} & \texttt{collab}  &\texttt{ddi}\\
          & Hits@100 & Hits@100 & Hits@100 & AUC &AUC & Hits@50 &Hits@20  \\
    \midrule
         \textbf{From Scratch}&  94.22\footnotesize{$\pm$1.64}&  94.31\footnotesize{$\pm$1.51}&  88.15\footnotesize{$\pm$0.38}&  99.11\footnotesize{$\pm$0.07} & 98.82\footnotesize{$\pm$0.21}  & 68.11\footnotesize{$\pm$0.54} &80.27\footnotesize{$\pm$3.98} \\
         \textbf{Predictor Only}&  92.34\footnotesize{$\pm$2.98}&  92.59\footnotesize{$\pm$2.60}&  85.26\footnotesize{$\pm$1.84}&  97.47\footnotesize{$\pm$0.79}& 96.81\footnotesize{$\pm$1.50}& 62.81\footnotesize{$\pm$2.62}&71.95\footnotesize{$\pm$5.68}\\
    \bottomrule
    \end{tabular}
\end{table*}

With a generalized heuristic for each dataset, there is no need to train a GNN and predictor from scratch. Instead, we can simply follow the generalized heuristic and train an MLP predictor only, which is significantly more efficient than training from scratch. We utilize a pretrained preprocessing block, which is a linear layer that transforms the dimension of node features into hidden channels of the GNN. 
We also train node embeddings on the OGB datasets to enhance the node representations.

The performance for training from scratch and training the MLP only is presented in Table~\ref{tab:performance_mlp}. Training the MLP only yields comparable performance to training from scratch, but with quicker convergence. This highlights the efficacy of the learned generalized heuristics, allowing us to achieve comparable results efficiently by training a simple MLP. This significantly improves efficiency and conserves computational resources.

\section{New Heuristics Derivation}
\label{para:heuristic_advance}
Existing heuristics are primarily handcrafted and may not be optimal for real-world graphs.
Leveraging the propagation operators, weight parameters and maximum order offers the potential to learn generalized, possibly more effective heuristics. We introduce two new heuristics below.

\begin{lemma}
Given $\mA^{(1)}=\Acs$, $\mA^{(2)}=\Ars$, $\beta^{(0)}=\beta^{(1)}=0$, $\beta^{(2)}=1$, $L=2$ within the $h(i,j)$ formulation, we derive a new local heuristic represented as 
\[ s(i, j)=\sum_{k \in \Gamma_i \cap \Gamma_j} \frac{1}{\tilde{d}_k^2}. \]
\end{lemma}
\begin{proof}
    From the given configurations and the $h(i,j)$ formulation, we obtain:
    \begin{equation}
    h(i, j) = (\Acs \Ars)_{i, j} = \sum_{k \in \mathcal{V}} \frac{\tilde{a}_{i k}}{\tilde{d}_k} \frac{\tilde{a}_{k j}}{\tilde{d}_k} = \sum_{k \in \Gamma_i \cap \Gamma_j} \frac{1}{\tilde{d}_k^2}.
    \end{equation}
    Consequently, the new local heuristic is $s(i, j) = \sum_{k \in \Gamma_i \cap \Gamma_j} \frac{1}{\tilde{d}_k^2}$.
\end{proof}

Our newly proposed heuristic draws parallels with both the Resource Allocation Index (RA)~\cite{RA} and the Adamic-Adar Index (AA)~\cite{AA}, yet introduces a larger penalty for high-degree nodes.

\begin{lemma}
Given $\mA^{(1)}=\mA^{(2)}=\Asym$, $\beta^{(0)}=\beta^{(1)}=0$, $\beta^{(2)}=1$, $L=2$ within the $h(i,j)$ formulation, we derive a new local heuristic represented as 
\begin{equation}
    s(i, j)=\frac{1}{\sqrt{\tilde{d}_i \tilde{d}_j}} \sum_{k \in \Gamma_i \cap \Gamma_j} \frac{1}{\tilde{d}_k}=\frac{s_{\mathrm{RA}}(i, j)}{\sqrt{\tilde{d}_i \tilde{d}_j}}.
\end{equation}
\end{lemma}
\begin{proof}
    From the given configurations and the $h(i,j)$ formulation, we obtain:
    \begin{equation}
    \begin{aligned}
        h(i, j)&=(\Asym^2)_{i, j}=\sum_{k \in \mathcal{V}} \frac{\tilde{a}_{i k}}{\sqrt{\tilde{d}_i \tilde{d}_k}} \frac{\tilde{a}_{k j}}{\sqrt{\tilde{d}_k \tilde{d}_j}}\\
        &=\frac{1}{\sqrt{\tilde{d}_i \tilde{d}_j}} \sum_{k \in \Gamma_i \cap \Gamma_j} \frac{1}{\tilde{d}_k}=\frac{s_{\mathrm{RA}}(i, j)}{\sqrt{\tilde{d}_i \tilde{d}_j}}.
    \end{aligned}
    \end{equation}
    Consequently, the new local heuristic is $s(i, j) = \frac{s_{\mathrm{RA}}(i, j)}{\sqrt{\tilde{d}_i \tilde{d}_j}}$.
\end{proof}

The RA Index, represented by either $ s_{\mathrm{RA}}(i, j) = (\Acs \A)_{i, j}$ or $s_{\mathrm{RA}}(i, j) = (\A \Ars)_{i, j}$, is unnormalized. However, based on our formulation, we present a symmetrically normalized version of the RA Index that potentially offers improved efficacy compared to the original RA Index.

Adjustments to the propagation operators and weight parameters also yield innovations in global heuristics. As one example, replacing the transition matrix $\Acs$ in the Random Walk with Restart (RWR)~\cite{PPR} with a symmetrically normalized matrix $\Asym$ births the Flow Propagation (FP)~\cite{FP} heuristic. A comprehensive list of novel global heuristics is beyond the scope of this section.

\section{Time complexity analysis}
\label{para:time_comp_analysis}
GCN~\cite{GCN}. The propagation cost is $\mathcal{O}(LMF)$ and the transformation cost with weight matrices $\boldsymbol{W}^{(l)}$ is $\mathcal{O}(LNF^2)$ if we do not modify feature dimensionality with $\boldsymbol{W}^{(l)}$. And the total cost is $\mathcal{O}(LF(M+NF))$.

GAT~\cite{GAT}. Denote the number of attention heads as $K$ and the average degree of nodes as $D$. Attention computation cost is $\mathcal{O}(NDF^2)$. The cost of computing one node's representation is $\mathcal{O}(ND^2F^2)$. For $K$ attention heads and $L$ layers, the overall time complexity is $\mathcal{O}(LKND^2F^2)$.

SEAL~\cite{SEAL}. Denote $E$ and $V$ as the average number of edges and vertices in the subgraphs. The complexity in constructing the enclosing subgraph is at most $\mathcal{O}(D^3)$, and the cost of computing shortest path is dependent on the algorithm, and the most efficient Dijkstras algorithm is $\mathcal{O}(V^2)$. The cost of the subgraph GNN is $\mathcal{O}(LEF)$. The algorithm need to be done for each edge, and the overall cost is $\mathcal{O}(M(V^2 + LEF))$.

NBFNet~\cite{NBFNet}. INDICATOR function takes $\mathcal{O}(NF)$ cost, MESSAGE function takes $\mathcal{O}(L(M+N)F)$ cost, and AGGREGATE function takes $\mathcal{O}(LNF^2)$ cost. The algorithm need to be done for each source node, and the total cost is $\mathcal{O}(LNF(M+NF))$.

Neo-GNN~\cite{Neo-GNN}. Denote the high-order matrix power as $L'$ and the average degree of nodes as $D$. Node structural feature computational cost is $\mathcal{O}(NDF^2)$, computing high-order matrix power up to $L'$ is $\mathcal{O}(L'N^2)$, computing output feature is $\mathcal{O}(MF)$. And Neo-GNN also needs a conventional GNN which needs a $\mathcal{O}(L(MF+NF^2))$ cost. The total cost is $\mathcal{O}(LMF+NDF^2)$.

BUDDY~\cite{BUDDY}. Denote the cost of hash operations as $H$. The preprocessing cost of BUDDY is $\mathcal{O}(LM(F+H))$. A link probability is computed by (i) extracting $L(L+2)$ structural features, which costs $\mathcal{O}(L^2 H)$; (ii) An MLP on structural and node features, which costs $\mathcal{O}(L^2 H+LF^2)$. The cost for computing probabilities for all links is $\mathcal{O}(LM(LH+F^2))$. The total cost is $\mathcal{O}(LM(LH+F^2))$.

\section{Limitation}
Our heuristic formulation does not aim to encompass all possible heuristics. A successful formulation generalizes heuristic commonalities and applies them for effective link prediction. Our formulation distills the critical characteristics of heuristics, specifically focusing on extracting common neighbors from local heuristics and global paths from global heuristics.

Upon careful consideration of the comprehensive survey~\cite{survey1,survey2}, our formulation may not contain certain heuristics with normalization operators like union, log, minimum, and maximum in the denominator. For example, heuristics like Jaccard and Adamic-Adar, due to their use of union and log operators in the denominator, cannot be directly represented as matrix multiplications. 
However, our formulation introduces normalization through the three propagation mechanisms and layer-specific learnable weights to adaptively control the degree of normalization, potentially mitigating the need for additional operators.

\end{document}